\documentclass{article}

\usepackage{hyperref}

\usepackage[accepted]{icml2023}

\usepackage[utf8]{inputenc} 
\usepackage[T1]{fontenc}    
\usepackage{hyperref}       
\usepackage{url}            
\usepackage{booktabs}       
\usepackage{amsfonts}       
\usepackage{nicefrac}       
\usepackage{microtype}      
\usepackage{xcolor}         
\usepackage{graphicx}
\usepackage{amsmath}
\usepackage{amssymb}
\usepackage{subcaption}
\usepackage{algorithm}
\usepackage{algorithmic}
\usepackage{multirow}
\usepackage{booktabs}
\usepackage{wrapfig}
\usepackage{hyperref}
\usepackage[toc,page]{appendix}
\usepackage{amsthm}
\newtheorem{theorem}{Theorem}[section]
\newtheorem{proposition}[theorem]{Proposition}

\renewcommand{\algorithmicrequire}{ \textbf{Input:}} 

\usepackage[capitalize,noabbrev]{cleveref}

\usepackage[textsize=tiny]{todonotes}

\icmltitlerunning{Mitigating Memorization of Noisy Labels by Clipping the Model Prediction}

\begin{document}

\twocolumn[
\icmltitle{Mitigating Memorization of Noisy Labels \\ by Clipping the Model Prediction}




\begin{icmlauthorlist}
\icmlauthor{Hongxin Wei}{sustech}
\icmlauthor{Huiping Zhuang}{scut}
\icmlauthor{Renchunzi Xie}{ntu}
\icmlauthor{Lei Feng}{ntu}
\icmlauthor{Gang Niu}{riken}
\icmlauthor{Bo An}{ntu}
\icmlauthor{Yixuan Li}{wisc}
\end{icmlauthorlist}

\icmlaffiliation{sustech}{Southern University of Science and Technology. Work done while working at UW-Madison as a visiting scholar.}
\icmlaffiliation{ntu}{Nanyang Technological University}
\icmlaffiliation{wisc}{University of Wisconsin-Madison}
\icmlaffiliation{scut}{South China University of Technology}
\icmlaffiliation{riken}{RIKEN AIP}

\icmlcorrespondingauthor{Hongxin Wei}{weihx@sustech.edu.cn}

\icmlkeywords{Noisy Labels, Logit Clipping}

\vskip 0.3in
]



\printAffiliationsAndNotice{}  

\begin{abstract}
In the presence of noisy labels, designing robust loss functions is critical for securing the generalization performance of deep neural networks. Cross Entropy (CE) loss has been shown to be \emph{not} robust to noisy labels due to its \emph{unboundedness}.
To alleviate this issue, existing works typically design specialized robust losses with the symmetric condition, which usually lead to the underfitting issue.
In this paper, our key idea is to induce a loss bound at the logit level, thus universally enhancing the noise robustness of existing losses.
Specifically, we propose \emph{logit clipping} (\textbf{LogitClip}), which clamps the norm of the logit vector to ensure that it is upper bounded by a constant. 
In this manner, CE loss equipped with our LogitClip method is effectively bounded, mitigating the overfitting to examples with noisy labels.  
Moreover, we present theoretical analyses to certify the noise-tolerant ability of LogitClip. 
Extensive experiments show that LogitClip not only significantly improves the noise robustness of CE loss, but also broadly enhances the generalization performance of popular robust losses.
\end{abstract}

\section{Introduction}
The success of supervised learning  relies heavily on a massive amount of data, where each training instance is labeled by a human annotator. However, labels solicited from humans can often be subject to label noise. The issue of noisy labels has been commonly observed in many real-world scenarios, such as crowdsourcing \citep{yan2014learning} and online queries \citep{blum2003noise}. As a result, models trained on such data containing noisy labels suffer from poor generalization performance \citep{pmlr-v70-arpit17a, zhang2016understanding}. This gives rise to the importance of \emph{noise-robust learning}, where the goal is to train a robust classifier in the presence of noisy and erroneous labels. The learning task thus provides stronger flexibility and practicality than the standard supervised learning, where each training example is provided with clean ground truth. 

Despite the most popular loss in classification tasks, CE loss has shown to be non-robust in the presence of label noise~\citep{ghosh2017robust, zhang2018generalized}, due to its \emph{unboundedness}. Concerningly, the loss could approach infinity when the observed noisy label mismatches the model's prediction. Consequently, the model would attempt to counteract the large loss by overfitting the label noise, leading to poor generalization performance. To bound the loss value, 
previous methods typically design robust losses with principal constraint, \emph{e.g.}, symmetric condition \citep{ghosh2017robust, ma2020normalized}. Despite their theoretical robustness, these specialized losses can cause difficulty in optimization, leading to underfitting issues on complex datasets \citep{zhang2018generalized, zhou2021asymmetric}.
This motivates our method, which mitigates the undesirable influence of unbounded loss without modifying the loss function.

In this paper, our key idea is to induce the loss bound at the logit level, which universally enhances the noise robustness of existing losses. 
Specifically, we propose \emph{logit clipping} (\textbf{LogitClip}), which clamps the norm of the logit vector to ensure that it is upper bounded by a constant. Our method can be interpreted as a constrained optimization with an inequality constraint placed on the logit vector. Theoretically, we show that CE loss equipped with our LogitClip method is always bounded. Consequently, the difference between the risks caused by the derived hypotheses under noisy and clean labels is always bounded. More importantly, the two bounds (Theorem~\ref{tm:sym_robust} and Theorem~\ref{tm:asym_robust}) depend on the logit norm threshold, where a smaller threshold induces tighter bounds. In this way, our theoretical analyses demonstrate the noise-tolerant ability of LogitClip.

To verify the effectiveness of our method, we conduct thorough empirical evaluations on both simulated and real-world noisy datasets, including CIFAR-10, CIFAR-100~\cite{krizhevsky2009learning}, and WebVision \citep{li2017webvision} datasets. The results demonstrate that logit clipping can significantly improve the noise-robustness of CE loss, under symmetric, asymmetric, instance-dependent, and real-world label noise. 
For example, on CIFAR-10 with instance-dependent label noise, LogitClip improves the test accuracy of CE loss from 68.36\% to 86.60\% -- a \textbf{18.24}\% of direct improvement. More importantly, we show that LogitClip can boost the performance of a wide range of popular robust loss functions, including MAE~\citep{ghosh2017robust}, PHuber-CE~\citep{menon2020can}, SCE~\citep{wang2019symmetric}, GCE~\citep{zhang2018generalized}, Taylor-CE~\citep{feng2020can}, NCE~\citep{ma2020normalized}, AEL, AUL~\citep{zhou2021asymmetric}, Cores~\citep{cheng2020learning}, and Active Passive losses \citep{ma2020normalized}.

We summarize our contributions as follows:
\begin{enumerate}
    \item We propose LogitClip -- a simple and effective method to enhance the noise robustness of existing losses. The key idea is to clamp the norm of the logit vector to bound the loss value, as shown in Proposition~\ref{prop:bound} and Proposition~\ref{prop:lip_condition}.
    \item We provide theoretical analyses in Theorem~\ref{tm:sym_robust} and Theorem~\ref{tm:asym_robust} to certify the noise-tolerant ability of our LogitClip, where a smaller threshold induces tighter bounds.
    \item  We conduct extensive evaluation to show that LogitClip can improve the robustness of CE and popular robust losses across various types of label noise. We empirically show that our method is model-agnostic and also applicable in large-scale real-world scenarios.
    \item We perform ablation studies that lead to improved understandings of our method. In particular, we contrast with alternative methods (\emph{e.g.}, RELU6~\citep{howard2017mobilenets}, Clipping-by-value) and demonstrate the advantages of our method with Clipping-by-norm.
\end{enumerate}

\section{Motivation and Method}
\label{sec:logitclip}

\subsection{Preliminaries: The Unboundedness of CE loss}
\label{sec:base}

In this work, we consider the multi-class classification task with $K$ different classes. Let $\mathcal{X}\subset\mathbb{R}^{d}$ be the input space and $\mathcal{Y} = \{1, \ldots, K\}$ be the label space, we consider a training dataset with $N$ samples $\{\boldsymbol{x}_i, y_i\}^N_{i=1}$, where $\boldsymbol{x}_i \in \mathcal{X}$ is the $i$-th instance sampled \emph{i.i.d.} from an underlying data-generating distribution $\mathcal{P}$  and $y_i \in \mathcal{Y}$ is the observed (and potentially noisy) label. A classifier is a function that maps from the input space to the label space $f: \mathcal{X} \rightarrow \mathbb{R}^{K}$ with trainable parameter $\boldsymbol{\theta} \in \mathbb{R}^{p}$.

Here, we consider \emph{composite} losses, which are comprised of a base loss function $\phi$ and an invertible link function $\sigma: \mathbb{R} \rightarrow[0,1]$. For example, the most commonly used composite loss in multi-class classification is Softmax Cross Entropy (CE) loss:
\begin{equation}
\begin{aligned}
\label{eq:ce}
\mathcal{L}_\text{CE}\left(f(\boldsymbol{x} ; \boldsymbol{\theta}), y\right) &=  -\sum_{j=1}^{K} \boldsymbol{y}_{j} \log (\sigma\left(\boldsymbol{z}_{j}\right)) \\
&=-\sum_{j=1}^{K} \boldsymbol{y}_{j} \log \left(\frac{e^{\boldsymbol z_{j}}}{\sum_{k=1}^{K}e^{\boldsymbol z_{k}}}\right),
\end{aligned}
\end{equation}

where $\boldsymbol{z}_{j} = f_{j}\left(\boldsymbol{x}; \boldsymbol{\theta}\right)$ corresponds to the $j$-th element of model output for the sample $\boldsymbol{x}$, and $\boldsymbol{y}_{j}$ is the $j$-th element of one-hot encoded label vector $\boldsymbol{y}$. Here $\sigma$ denotes the softmax function, which is also the invertible link function. As an \emph{unbounded} loss function, CE is shown to be non-robust in the presence of label noise \citep{ghosh2017robust, zhang2018generalized}, since the observed labels might be incorrect. 
In particular, the gradients of CE can be shown as: 
$$\frac{\partial \mathcal{L}_{\mathrm{CE}}(f(\boldsymbol{x}, \boldsymbol{\theta}), y)}{\partial \boldsymbol{\theta}}=-\frac{1}{\sigma_{y}(\boldsymbol{z}))} \nabla_{\boldsymbol{\theta}} \sigma_{y}(\boldsymbol{z})),$$
From the equation, we find that CE pays more attention to those examples with lower confidences, i.e., hard examples (or noisy examples).
As $\sigma({z}) \to 0$, the unbounded loss would approach infinity, leading to severe overfitting issues on noisy labels. 


\subsection{Our Proposed Method}
\label{sec:method}

In this paper, we propose a general strategy that can make the loss function noise-robust, avoiding the inherent drawback of overfitting the label noise. Our key idea is to \emph{bound} the logit value in the link function. Our method is motivated by the following reformulation of the softmax CE loss:
\begin{equation}
\begin{aligned}
\label{eq:ce_logit}
    \mathcal{L}_{\mathrm{CE}}(\boldsymbol{z}, y) 
    &= \log (1+\sum_{j \neq y}e^{\boldsymbol{z}_j-\boldsymbol{z}_y}) \\
    &\leq \log (1+(K-1)\cdot e^{\boldsymbol{z}_{\max}-\boldsymbol{z}_{\min}}), 
\end{aligned}
\end{equation}

where $\boldsymbol{z}^{\max}$ and $\boldsymbol{z}^{\min}$ denote the maximum and minimum value in the logit vector $\boldsymbol{z} = f(\boldsymbol{x};\boldsymbol{\theta})$. The above formulation suggests that, if $\boldsymbol{z}^{\max} - \boldsymbol{z}^{\min}$ is upper bounded, $\mathcal{L}_{\mathrm{CE}}$ could not reach infinity, thereby preventing the model from overfitting to examples with noisy labels. 

To enforce the upper bound, we propose to bound the logit values by norm, which preserves the direction of the original logit vector. The training objective can be formalized as a constrained optimization with inequality constraint:
\begin{align*}
&\text{minimize} \quad \mathbb{E}_{(\boldsymbol{x},y)\sim\mathcal{P}}\left[\mathcal{L}_{\text{CE}}\left(f(\boldsymbol{x} ; {\boldsymbol\theta}), y\right)\right] \\
&\text{subject to} \quad \left\|f(\boldsymbol{x} ; {\boldsymbol\theta})\right\|_{p}\leq\alpha,
\end{align*}
where $\|\cdot \|_{p}$ denotes the $p$-norm (also called $\ell_p$-norm, $p \geq 1$) of the logit vector. However, performing constrained optimization in the context of modern neural networks is non-trivial, as explicitly shown in Appendix~\ref{app:normreg}. To circumvent the issue, we convert the objective into an alternative loss function that can be end-to-end trainable, strictly enforcing an upper bound of vector norm.

\paragraph{Logit Clipping.} We propose \emph{logit clipping} (dubbed LogitClip), which clamps the norm of the logit vector to ensure it is upper bounded by a constant. Formally, the new link function is defined as:
\begin{align}
\label{eq:logitclip}
\bar{\sigma}_{\tau}(\boldsymbol{z}) \doteq \sigma(\operatorname{clip}_{\tau}(\boldsymbol{z})), \ \operatorname{clip}_{\tau}(\boldsymbol{z}) \doteq \begin{cases}\tau \cdot \frac{\boldsymbol{z}}{\|\boldsymbol{z}\|_{p}} & \text { if }\|\boldsymbol{z}\|_{p} \geq \tau \\ \boldsymbol{z} & \text { else }\end{cases},
\end{align}
where $\tau$ is the upper bound of the norm. 
Our method ensures that: (1) the norm of the clamped logit vector $\operatorname{clip}_{\tau}(\boldsymbol{z})$ is bounded by $\tau$, and (2) the clamped logit vector preserves the same direction (and prediction) as the original logit vector. To increase flexibility, one can set the scale factor $\delta$ to a value that differs from the threshold $\tau$. In this form, the clipping function can be represented as:
\begin{align}
\label{eq:clipdiff}
\operatorname{clip}_{\tau, \delta}(\boldsymbol{z}) \doteq \begin{cases}\delta \cdot \frac{\boldsymbol{z}}{\|\boldsymbol{z}\|_{p}} & \text { if }\|\boldsymbol{z}\|_{p} \geq \tau \\ \boldsymbol{z} & \text { else }\end{cases}.
\end{align}
Note that LogitClip can be compatible with various loss functions, as we will later demonstrate in Section~\ref{sec:experiments}. In other words, we can employ LogitClip in the link function $\sigma$ of the composite losses, where the base loss function $\phi$ can be CE or other existing robust loss functions. Taking cross-entropy loss as an example, the new training objective now becomes:
\begin{align*}
&\text{minimize} \quad \mathbb{E}_{(\boldsymbol{x},y)\sim\mathcal{P}}\left[\mathcal{L}_{\text{CE}}\left(\operatorname{clip}_{\tau}(f(\boldsymbol{x} ; {\boldsymbol\theta})), y\right)\right].
\end{align*}

In what follows, we provide a formal analysis on the bound of the new loss function. For convenience, we denote CE loss with LogitClip as $\mathcal{L}^{\tau}_{\mathrm{CE}}$. Without loss of generality, we use the same value for the scale factor and the threshold for simplicity (as shown in Equation~\ref{eq:logitclip}). We start from the case of max norm ($p=\infty$), i.e., $-\tau \leq \operatorname{clip}_{\tau}(\boldsymbol{z}_{j}) \leq \tau$. Then we can derive an upper bound and a lower bound of $\mathcal{L}^{\tau}_{\mathrm{CE}}$.

\begin{proposition}[Upper and Lower Bounds of CE with LogitClip]
\label{prop:bound}
For any input $\boldsymbol{x}$ and any positive number $\tau \in \mathbb{R}^{+}$, CE loss with LogitClip defined in Eq.~(\ref{eq:logitclip}) has a lower bound and an upper bound: 
$$ \log (1 + (K - 1) \cdot e^{-2\tau}) \leq \mathcal{L}^{\tau}_{\mathrm{CE}}\left(\boldsymbol{z}, y\right)\leq \log (1 + (K - 1) \cdot e^{2\tau}).$$
\end{proposition}

The proof of Proposition~\ref{prop:bound} is provided in Appendix~\ref{app:proofs_1}. 
Through Proposition~\ref{prop:bound}, we show that cross-entropy loss equipped with LogitClip is bounded. The conclusion can be extended to other norms, since $\|\boldsymbol{z}\|_p \leq \|\boldsymbol{z}\|_q \leq \|\boldsymbol{z}\|_{\infty}$ for $p \geq q$. To provide a straightforward view, we show in Figure~\ref{fig:bound} how the hyperparameter $\tau$ and the class num $K$ affect the upper and lower bounds of CE with LogitClip. When $\tau \to \infty$, we have $0 \leq \mathcal{L}^{\tau}_{\mathrm{CE}} \leq \infty$, which is equivalent to the original loss $\mathcal{L}_{\mathrm{CE}}$ in Equation~\ref{eq:ce}. On the other hand, if $\tau \to 0$, the lower bound would be close to the upper bound, which may result in difficulties for loss optimization. We will analyze the effect of $\tau$ in detail in Section~\ref{sec:experiments}.

\begin{figure}[!t]
    \centering
    \includegraphics[width=0.40\textwidth]{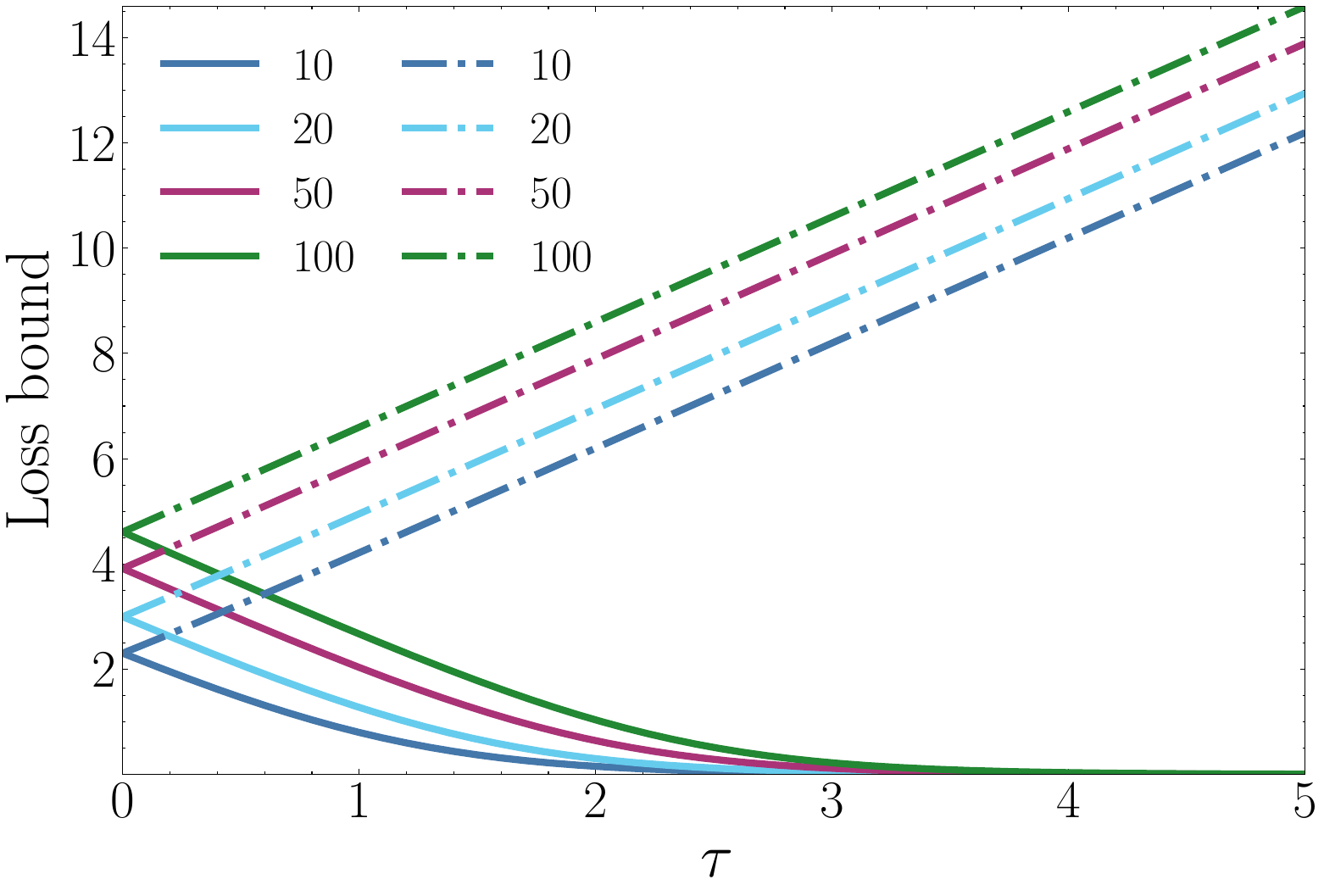}
    \caption{The effect of $\tau$ and $K$ on the Loss bound. The dashed lines denote the upper bounds and the solid lines show the lower bounds. Four colors are used to present bounds with various values of $K$.}
    \label{fig:bound}
\end{figure}

Based on proposition~\ref{prop:bound}, we further analyze the noise robustness of $\mathcal{L}^{\tau}_{\mathrm{CE}}$ with LogitClip. We denote the clean ground-truth label of $\boldsymbol{x}$ as $y^{\star}$. Here, we follow the most common setting where label noise is instance-independent \citep{ghosh2017robust, feng2020can, ma2020normalized}. Under this assumption, label noise can be either \emph{symmetric} (i.e., uniform) or \emph{asymmetric} (i.e., class-conditional). Let $\eta \in [0,1]$ be the overall noise rate and $\eta_{jk}$ be the class-wise noise rate from ground-truth class $j$ to class $k$, where $\eta_{jk} = p\left(y=k \mid y^{\star}=j\right)$. For symmetric noise, $\eta_{jk} = \frac{\eta}{K-1}$ for $j \neq k$ and $\eta_{jk} = 1-\eta$ for $j=k$. For asymmetric noise, $\eta_{jk}$ is conditioned on both the true class $j$ and the mislabeled class $k$. Given any classifier $f$ and loss function $\mathcal{L}$, the risk of $f$ under clean labels is defined as: $\mathcal{R}_{\mathcal{L}}(f)=\mathbb{E}_{(\boldsymbol{x}, y^\star)\sim \mathcal{P}_\text{clean}}[\mathcal{L}(f(\boldsymbol{x}), y^{\star})]$ and the risk under label noise rate $\eta$ is: $\mathcal{R}_{\mathcal{L}}^\eta(f)=\mathbb{E}_{(\boldsymbol{x}, y)\sim \mathcal{P}_\text{noisy}^\eta}[\mathcal{L}(f(\boldsymbol{x}), y)]$. Let $\tilde{f}$ and $f^{\star}$ be the global minimizers of $\mathcal{R}_{\mathcal{L}^{\tau}_{\mathrm{CE}}}^\eta(f)$ and $\mathcal{R}_{\mathcal{L}^{\tau}_{\mathrm{CE}}}(f)$, respectively.

\begin{theorem}
\label{tm:sym_robust}
Under symmetric label noise with $\eta \leq 1-\frac{1}{K}$,
$$
 0 \leq \mathcal{R}_{\mathcal{L}^{\tau}_{\mathrm{CE}}}(\tilde{f}) - \mathcal{R}_{\mathcal{L}^{\tau}_{\mathrm{CE}}}(f^{\star}) \leq \frac{\eta K}{(1-\eta)K-1} \cdot A^{K}_{\tau},
$$
where $A^{K}_{\tau} = \log\left(\frac{1+(K-1)e^{2\tau}}{1+(K-1)e^{-2\tau}}\right)$ is a constant that depends on $\tau$ and number of classes $K$.
\end{theorem}

\begin{theorem}
\label{tm:asym_robust}
Under asymmetric label noise with $\eta_{i j}<$ $1-\eta_i, \forall j \neq i, \forall i, j \in[k]$, where $\eta_{i j}=p(y=j \mid y^{\star}=$ $i), \forall j \neq i$ and $\left.\left(1-\eta_i\right)=p(y=i \mid y^{\star}=i)\right)$, then
$$
0 \leq \mathcal{R}_{\mathcal{L}^{\tau}_{\mathrm{CE}}}^\eta\left(f^*\right)-\mathcal{R}_{\mathcal{L}^{\tau}_{\mathrm{CE}}}^\eta(\tilde{f}) \leq B^{K}_{\tau},
$$
where $B^{K}_{\tau} = K\log\left(\frac{1+(K-1)e^{2\tau}}{1+(K-1)e^{-2\tau}}\right) \mathbb{E}_{(\boldsymbol{x}, y^{\star})\sim \mathcal{P}_\text{clean}}\left(1-\eta_i\right)>0$.
\end{theorem}

The proofs of the above two Theorems are provided in Appendix~\ref{app:proofs_2} and Appendix~\ref{app:proofs_3}, respectively.
Theorem~\ref{tm:sym_robust} and Theorem~\ref{tm:asym_robust} show that with LogitClip, the difference of the risks caused by the derived hypotheses $\tilde{f}$ and $f^{\star}$ under noisy and clean labels is always bounded. More specifically, the two bounds depend on the parameter $\tau$. With a smaller $\tau$, both the $A^{K}_{\tau}$ in Theorem~\ref{tm:sym_robust} and the $B^{K}_{\tau}$ in Theorem~\ref{tm:asym_robust} become smaller, indicating tighter bounds. The above analysis provably demonstrates the noise-tolerant ability of cross-entropy loss with LogitClip method. \textbf{We extend our analysis to an instance-dependent setting in Appendix \ref{app:theory_dependent}}. We proceed by a general analysis of composite losses with LogitClip.

\vspace{5pt}
\begin{proposition}
\label{prop:lip_condition}
Given any base loss $\phi(x)$ that satisfies the Lipschitz condition with constant $L$ on the domain $ M^{K}_{\tau} \leq x \leq N^{K}_{\tau}$, the resulting composite loss with $\tilde{\sigma}_{\tau}$ defined in Equation~$\ref{eq:logitclip}$ is bounded:
$$
\left|\mathcal{L}^{\tau}_{\phi}\left(f(\boldsymbol{x} ; {\boldsymbol\theta}), y\right)\right| \leq L\left(N^{K}_{\tau} - M^{K}_{\tau}\right) + \left|\phi(M^{K}_{\tau})\right|,
$$
where $M^{K}_{\tau} = \frac{1}{1+(K-1)\cdot e^{2\tau}}$ and $N^{K}_{\tau} = \frac{1}{1+(K-1)\cdot e^{-2\tau}}$.
\end{proposition}

The proofs of Proposition~\ref{prop:lip_condition} are provided in Appendix~\ref{app:proofs_4}. From Proposition~\ref{prop:lip_condition}, we show that composite losses equipped with LogitClip are bounded if their base losses are locally Lipschitz continuous \citep{sohrab2003basic} on the domain, which depends on $\tau$ and class number $K$. If $\tau \to \infty$, the domain turns to be $0 \leq x \leq 1$. In this case, CE and Focal loss \citep{lin2017focal} do not satisfy the Lipschitz condition and thus are unbounded. Otherwise, with a constant $\tau < \infty$, CE and Focal loss are locally Lipschitz continuous on the domain so that they can be bounded with our LogitClip. Similarly, this conclusion is also applicable to existing robust losses, which generally satisfy the Lipschitz condition. Based on the loss bound in Proposition~\ref{prop:lip_condition}, we can easily derive the bounds of $\mathcal{R}_{\mathcal{L}^{\tau}_{\phi}}(\tilde{f}) - \mathcal{R}_{\mathcal{L}^{\tau}_{\phi}}(f^{\star})$ under symmetric and asymmetric label noise, following our proofs for Theorem \ref{tm:sym_robust} and Theorem \ref{tm:asym_robust}. Overall, the above analysis provably shows that LogitClip can enable the resulting composite loss to be noise-tolerant, with the locally Lipschitz condition for the base loss. We will further verify our analysis experimentally in Section~\ref{sec:experiments}.

\begin{table*}[!t]
\centering
\caption{Average test accuracy (\%) with standard deviation on CIFAR-10 under various types of noisy labels (over 5 runs). The bold indicates the improved results by integrating LogitClip (LC). }
\label{tab:cifar10_results}
\renewcommand\arraystretch{0.85}
\resizebox{0.95\textwidth}{!}{
\setlength{\tabcolsep}{4mm}{
\begin{tabular}{cccccc}
\toprule
Method & Sym-20\% & Sym-50\% & Asymmetric & Dependent & Real  \\
\midrule
 CE& 86.73$\pm$0.72  &  70.88$\pm$0.46 & 78.34$\pm$0.54 & 68.26$\pm$0.21 & 72.85$\pm$0.32\\
\textbf{+ LC (Ours)}& \textbf{91.62$\pm$0.16} & \textbf{84.37$\pm$0.34} & \textbf{86.91$\pm$0.68} & \textbf{86.74$\pm$0.55} & \textbf{82.06$\pm$0.70}\\
\midrule
Focal & 87.17$\pm$0.68  &  70.61$\pm$0.59 & 79.61$\pm$0.40 & 69.40$\pm$0.55 & 72.29$\pm$0.41\\
\textbf{+ LC (Ours)}& \textbf{91.91$\pm$0.46} &  \textbf{84.65$\pm$0.74} & \textbf{85.42$\pm$0.80} & \textbf{87.02$\pm$0.22} & \textbf{81.78$\pm$0.37}\\
\midrule
MAE & 88.92$\pm$0.36  &  75.73$\pm$1.15 & 56.74$\pm$0.71 & 53.92$\pm$1.07 & 53.26$\pm$0.67\\
\textbf{+ LC (Ours)}& \textbf{90.84$\pm$0.20} &  \textbf{86.06$\pm$0.52} & \textbf{83.74$\pm$0.46} & \textbf{87.76$\pm$0.79} & \textbf{82.83$\pm$0.38}\\
\midrule
PHuber-CE& 90.92$\pm$0.93  &  74.07$\pm$0.41 & 81.26$\pm$0.65 & 75.07$\pm$0.26 & 76.61$\pm$0.58\\
\textbf{+ LC (Ours)}& \textbf{91.90$\pm$0.65} &  \textbf{84.64$\pm$0.19} & \textbf{85.54$\pm$0.38} & \textbf{86.96$\pm$0.72} & \textbf{82.41$\pm$0.82}\\
\midrule
SCE& 91.48$\pm$0.78  &  85.38$\pm$0.47 & 78.65$\pm$0.60 & 87.05$\pm$0.58 & 81.65$\pm$0.35\\
\textbf{+ LC (Ours)}& 91.56$\pm$0.14 & \textbf{86.18$\pm$0.35} & \textbf{84.47$\pm$1.04} & \textbf{87.61$\pm$0.19} & \textbf{82.43$\pm$0.37}\\
\midrule
GCE& 90.58$\pm$0.66  &  85.51$\pm$0.45 & 79.35$\pm$0.58 & 87.64$\pm$0.62 & 81.38$\pm$0.70\\
\textbf{+ LC (Ours)}& \textbf{91.21$\pm$0.52} & 85.90$\pm$0.15 & \textbf{84.24$\pm$0.84} & 87.69$\pm$0.21 & \textbf{82.44$\pm$0.11}\\
\midrule
Taylor-CE & 90.44$\pm$0.40  &  85.71$\pm$0.26 & 80.92$\pm$1.37 & 87.20$\pm$0.98 & 82.32$\pm$1.12\\
\textbf{+ LC (Ours)}& \textbf{91.37$\pm$0.30} &  \textbf{86.31$\pm$0.18} & \textbf{84.57$\pm$0.45} & \textbf{88.05$\pm$0.72} & \textbf{82.86$\pm$0.55}\\
\midrule
NCE & 90.87$\pm$0.94  &  68.45$\pm$0.43 & 83.68$\pm$0.85 & 73.53$\pm$0.93 & 79.96$\pm$0.25\\
\textbf{+ LC (Ours)}& \textbf{91.70$\pm$0.27} & \textbf{85.88$\pm$0.89} & \textbf{88.44$\pm$0.53} & \textbf{87.59$\pm$1.21} & \textbf{82.11$\pm$0.64}\\
\midrule
AEL & 88.59$\pm$1.03  &  77.48$\pm$0.88 & 60.90$\pm$1.27 & 84.55$\pm$0.71 & 69.40$\pm$0.38\\
\textbf{+ LC (Ours)}& \textbf{90.58$\pm$0.36} & \textbf{86.07$\pm$0.43} & \textbf{82.12$\pm$0.89} & \textbf{87.77$\pm$0.48} & \textbf{82.17$\pm$0.71}\\
\midrule
AUL & 76.73$\pm$0.33  &  75.27$\pm$0.93 & 59.80$\pm$0.75 & 73.66$\pm$0.45 & 63.96$\pm$1.15\\
\textbf{+ LC (Ours)}& \textbf{91.31$\pm$0.85} &  \textbf{85.98$\pm$0.74} & \textbf{84.38$\pm$0.53} & \textbf{87.93$\pm$0.86} & \textbf{82.46$\pm$0.66}\\
\midrule
Cores & 91.56$\pm$0.23  &  85.32$\pm$0.36 & 85.30$\pm$0.63 & 86.65$\pm$0.83 & 82.28$\pm$0.28\\
\textbf{+ LC (Ours)}& 91.72$\pm$0.15 &  \textbf{86.06$\pm$0.28} & \textbf{86.36$\pm$0.33} & \textbf{87.75$\pm$0.09} & \textbf{82.83$\pm$0.40}\\
\bottomrule
\end{tabular}
}
}

\end{table*}

\section{Experiments}
\label{sec:experiments}
In this section, we validate the effectiveness of our method on three benchmarks, including simulated and real-world datasets under various types of label noise. We show that LogitClip not only significantly improves the robustness of CE loss, but also broadly enhances the performance of popular robust losses. In addition, we perform a sensitivity analysis to validate the effect of $\tau$.

\subsection{Setups}
\paragraph{Datasets.} To verify the efficacy of LogitClip, we comprehensively consider four different types of label noise, including \textbf{(1)} symmetric noise, \textbf{(2)} asymmetric noise \citep{zhang2018generalized}, \textbf{(3)} instance-dependent noise \citep{chen2020beyond}, and \textbf{(4)} real-world noise on CIFAR-10/100 \citep{krizhevsky2009learning} and WebVision \citep{li2017webvision} datasets. For symmetric noise, each label can be flipped to any other class with the same probability. In our experiments, we uniformly flip the label to other classes with a probability of 20\% and 50\%, respectively. For asymmetric noise, the labels might be only flipped to similar classes \citep{patrini2017making, zhang2018generalized}. In our CIFAR-10 experiments, we generate asymmetric noisy labels by mapping \textsc{truck} $\rightarrow$ \textsc{automobile}, \textsc{bird} $\rightarrow$ \textsc{airplane}, \textsc{deer} $\rightarrow$ \textsc{horse}, and \textsc{cat} $\leftrightarrow$ \textsc{dog} with probability 40\%. For CIFAR-100, we flip each class into the next circularly with a probability of 40\%. For instance-dependent noise, we assume the mislabeling probability of each instance is dependent on the corresponding input features \citep{chen2020beyond, xia2020parts}. In the experiments, we use the instance-dependent noise from PDN \citep{xia2020parts} with a noisy rate of 40\%, where the noise is synthesized based on the DNN prediction error. For real-world noisy labels on the CIFAR datasets, we use the ``Worst" label set of CIFAR-10N and the ``Fine" label set of CIFAR-100N \citep{wei2022learning}, respectively.


\paragraph{Training details.} We perform training with WRN-40-2 \citep{zagoruyko2016wide} on CIFAR-10 and CIFAR-100. In particular, we train the network for 200 epochs using SGD with a momentum of 0.9, a weight decay of 0.0005, and a batch size of 128. We set the initial learning rate as 0.1, and reduce it by a factor of 10 after 80 and 140 epochs. For our LogitClip in all experiments, we set $\delta = 1/\tau$ (see Equation~\ref{eq:clipdiff}) and use Euclidean norm, i.e., $p=2$. We use 5k noisy samples as the validation dataset to tune the hyperparameter $1/\tau$ in $\{0.1, 0.5, 1, 1.5, \ldots, 4.5, 5\}$, then train the model on the full training set and report the average test accuracy in the last 10 epochs. We repeat all experiments 5 times with different random seeds. More training details are described in Appendix~\ref{app:exp_setup}.

\subsection{CIFAR-10 and CIFAR-100}

On CIFAR-10 and CIFAR-100, we validate that LogitClip can enhance the noise robustness of existing loss functions. In particular, we consider the following loss functions: \textbf{(1)} Cross-Entropy (CE), which is the most commonly used classification loss. \textbf{(2)} Focal loss, which is originally proposed for dense object detection and also an unbounded classification loss function, $\mathcal{L}_\mathrm{Focal}\left(\boldsymbol{p}, y\right) =  -\sum_{j=1}^{K} \boldsymbol{y}_{j} (1-\boldsymbol{p}_{j})^{\gamma} \log (\boldsymbol{p}_{j})$. We set $\gamma=0.5$ in our experiments. \textbf{(3)} Mean absolute error (MAE) \citep{ghosh2017robust}, a symmetric loss function  $\mathcal{L}_\mathrm{MAE}\left(\boldsymbol{p}, y\right) = \left\|\boldsymbol{y}_{j}-\boldsymbol{p}_j\right\|_{1}$ that has been demonstrated to be robust to label noise. \textbf{(4)} PHuber-CE \citep{menon2020can}, a loss variant of gradient clipping for learning with noisy labels. \textbf{(5)} SCE \citep{wang2019symmetric}, which boosts CE symmetrically with a noise-robust counterpart Reverse Cross Entropy (RCE). \textbf{(6)} GCE \citep{zhang2018generalized}, a bounded loss function that uses a hyperparameter $q$ to balance between MAE and CE. Following the recommended setting in the corresponding paper, we set the hyperparameter $q$ as 0.7. \textbf{(7)} Taylor-CE \citep{feng2020can}, which controls the order of the Taylor Series to balance between MAE and CE. \textbf{(8)} NCE \citep{ma2020normalized}, which employs loss normalization to boost the robustness of CE loss. \textbf{(9)} AEL, AUL, and AGCE \citep{zhou2021asymmetric}, which are asymmetric loss functions. \textbf{(10)} Cores \citep{cheng2020learning}, a robust loss that is guaranteed to be robust to instance-dependent label noise. We also consider the Active Passive Loss \citep{ma2020normalized} by including NCE+MAE and NCE+AGCE.

\paragraph{Can LogitClip improve the noise-robustness of existing loss functions?} Table \ref{tab:cifar10_results} and Table \ref{tab:cifar100_results} present the average test accuracy of models trained with different noise-robust loss functions on CIFAR-10 and CIFAR-100, under various types of noisy labels. A salient observation is that our method drastically improves the noise-robustness performance of CE by employing LogitClip. For example, on the CIFAR-10 with instance dependent label noise, our approach improves the test accuracy of CE loss from 68.36\% to 86.60\% -- a \textbf{18.24}\% of direct improvement. On CIFAR-100, our method also improves performance by a significant margin. More importantly, we show that the LogitClip can boost performance for a wide range of loss functions, including non-robust and robust losses. For example, we observe that, on the CIFAR-10 with asymmetric label noise, the test accuracy of the NCE loss is improved to 88.44\% when employing LogitClip, establishing strong robustness against all types of label noise. In addition to loss functions, we show our method can also enhance other deep learning methods in Appendices~\ref{app:dividemix} and \ref{app:sota}.

\begin{table*}[!t]
\centering
\caption{Average test accuracy (\%) with standard deviation on CIFAR-100 under various types of noisy labels (over 5 runs). The bold indicates the improved results by integrating our method. }
\label{tab:cifar100_results}
\renewcommand\arraystretch{0.85}
\resizebox{0.95\textwidth}{!}{
\setlength{\tabcolsep}{4mm}{
\begin{tabular}{cccccc}
\toprule
Method & Sym-20\% & Sym-50\% & Asymmetric & Dependent & Real  \\
\midrule
 CE& 64.81$\pm$1.10 &  47.07$\pm$1.07 & 47.68$\pm$0.93 & 52.49$\pm$0.79 & 55.68$\pm$0.81\\
\textbf{+ LC (Ours)}& \textbf{71.59$\pm$0.76} &  \textbf{63.16$\pm$0.74} & \textbf{59.04$\pm$0.18} & \textbf{66.24$\pm$0.71} & \textbf{58.61$\pm$0.35}\\
\midrule
Focal & 64.76$\pm$0.14  &  47.06$\pm$0.56 & 48.59$\pm$0.73 & 52.87$\pm$0.57 & 55.01$\pm$0.65\\
\textbf{+ LC (Ours)}& \textbf{71.39$\pm$0.79} &  \textbf{62.91$\pm$0.25} & \textbf{59.53$\pm$0.76} & \textbf{66.38$\pm$0.30} & \textbf{58.76$\pm$0.23}\\
\midrule
PHuber-CE& 71.47$\pm$0.29 &  60.52$\pm$0.67 & 47.26$\pm$0.44 & 64.33$\pm$0.41 & 56.18$\pm$0.59\\
\textbf{+ LC (Ours)}& 71.89$\pm$0.31 &  \textbf{61.46$\pm$0.75} & \textbf{53.95$\pm$0.38} & \textbf{65.08$\pm$0.17} & \textbf{58.64$\pm$0.49} \\
\midrule
SCE&  70.11$\pm$0.31 &  58.56$\pm$0.78 & 44.91$\pm$0.62 & 62.86$\pm$0.74 & 58.27$\pm$0.88\\
\textbf{+ LC (Ours)}& \textbf{71.19$\pm$0.23} &  \textbf{60.11$\pm$0.15}  & \textbf{58.67$\pm$0.84}  & \textbf{64.76$\pm$0.28}  & \textbf{59.23$\pm$0.69}  \\
\midrule
GCE& 63.30$\pm$0.48 &  9.10$\pm$0.72 & 40.40$\pm$0.45 & 27.45$\pm$0.50 & 49.54$\pm$0.58\\
\textbf{+ LC (Ours)}& \textbf{70.22$\pm$0.52} &  \textbf{62.14$\pm$1.20} & \textbf{54.41$\pm$0.64} & \textbf{66.25$\pm$0.85} & \textbf{58.77$\pm$0.76}\\
\midrule
Cores & 69.97$\pm$0.56 &  55.37$\pm$0.84 & 50.24$\pm$0.38 & 59.85$\pm$0.61 & 56.49$\pm$0.53\\
\textbf{+ LC (Ours)}& \textbf{71.67$\pm$0.26} &  \textbf{62.67$\pm$0.33} & \textbf{63.32$\pm$0.70} & \textbf{66.31$\pm$0.25} & \textbf{59.23$\pm$0.34}\\
\midrule
NCE+MAE & 70.55$\pm$0.83 &  61.01$\pm$0.94 & 53.68$\pm$0.18 & 65.02$\pm$0.42 & 59.27$\pm$0.12\\
\textbf{+ LC (Ours)}& \textbf{71.59$\pm$0.65} &  \textbf{62.85$\pm$0.52} & \textbf{54.51$\pm$0.73} & \textbf{66.58$\pm$0.18} & \textbf{60.08$\pm$0.34}\\
\midrule
NCE+AGCE & 69.69$\pm$0.30 &  58.13$\pm$0.43 & 58.17$\pm$0.25 & 64.35$\pm$0.39 & 58.64$\pm$0.65\\
\textbf{+ LC (Ours)}& \textbf{71.30$\pm$0.48} & \textbf{63.55$\pm$0.45} & \textbf{59.27$\pm$0.32} & \textbf{65.51$\pm$0.47} & \textbf{59.57$\pm$0.55}\\

\bottomrule

\end{tabular}
}
}
\end{table*}

\begin{figure*}[!t]
    \centering
    \begin{subfigure}[b]{0.48\textwidth}
        \centering
        \includegraphics[height=4.0cm,width=6cm]{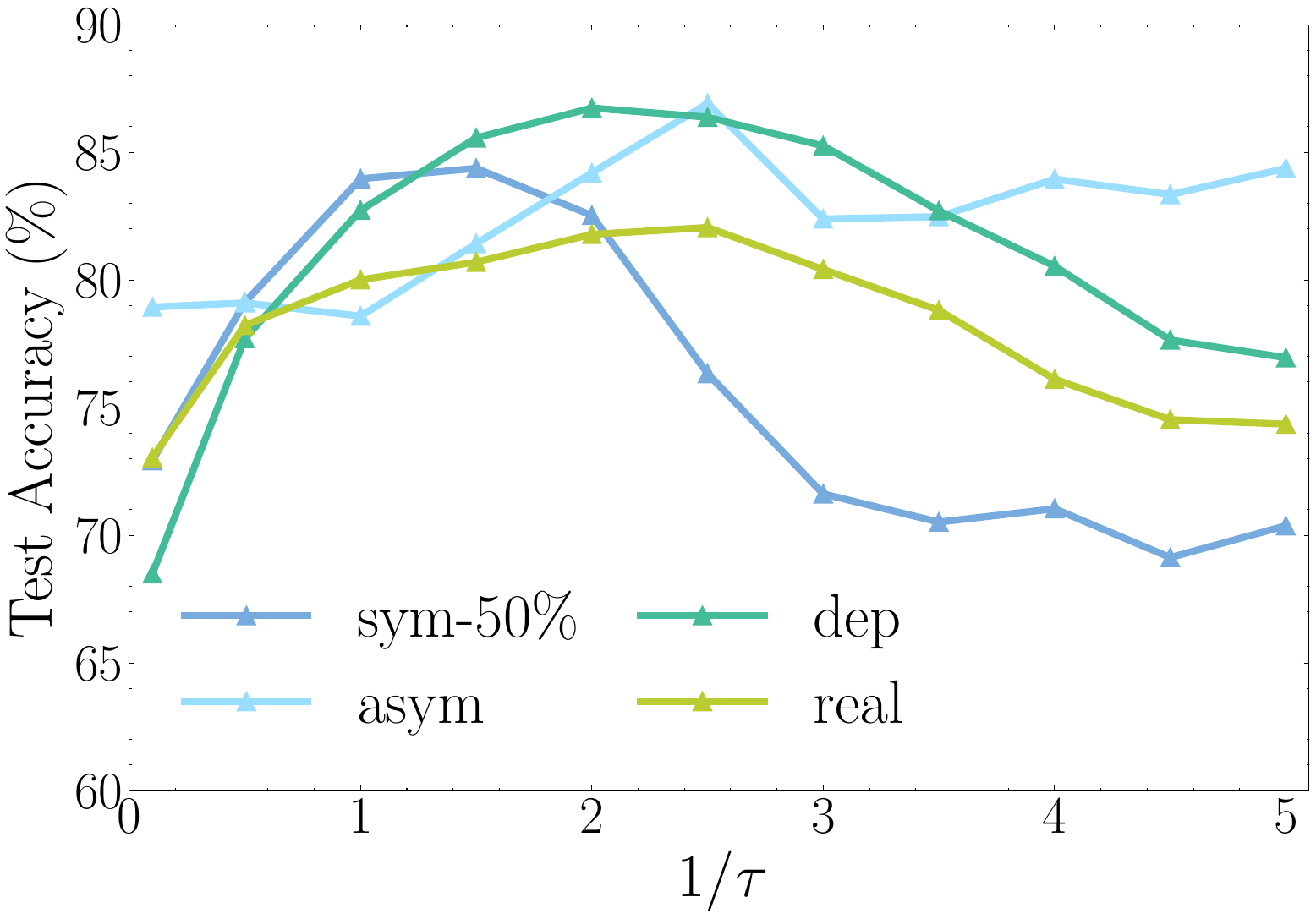}
        \caption{CIFAR-10}
        \label{fig:cifar10_tau}
    \end{subfigure}
    \begin{subfigure}[b]{0.48\textwidth}
        \centering
        \includegraphics[height=4.0cm,width=6cm]{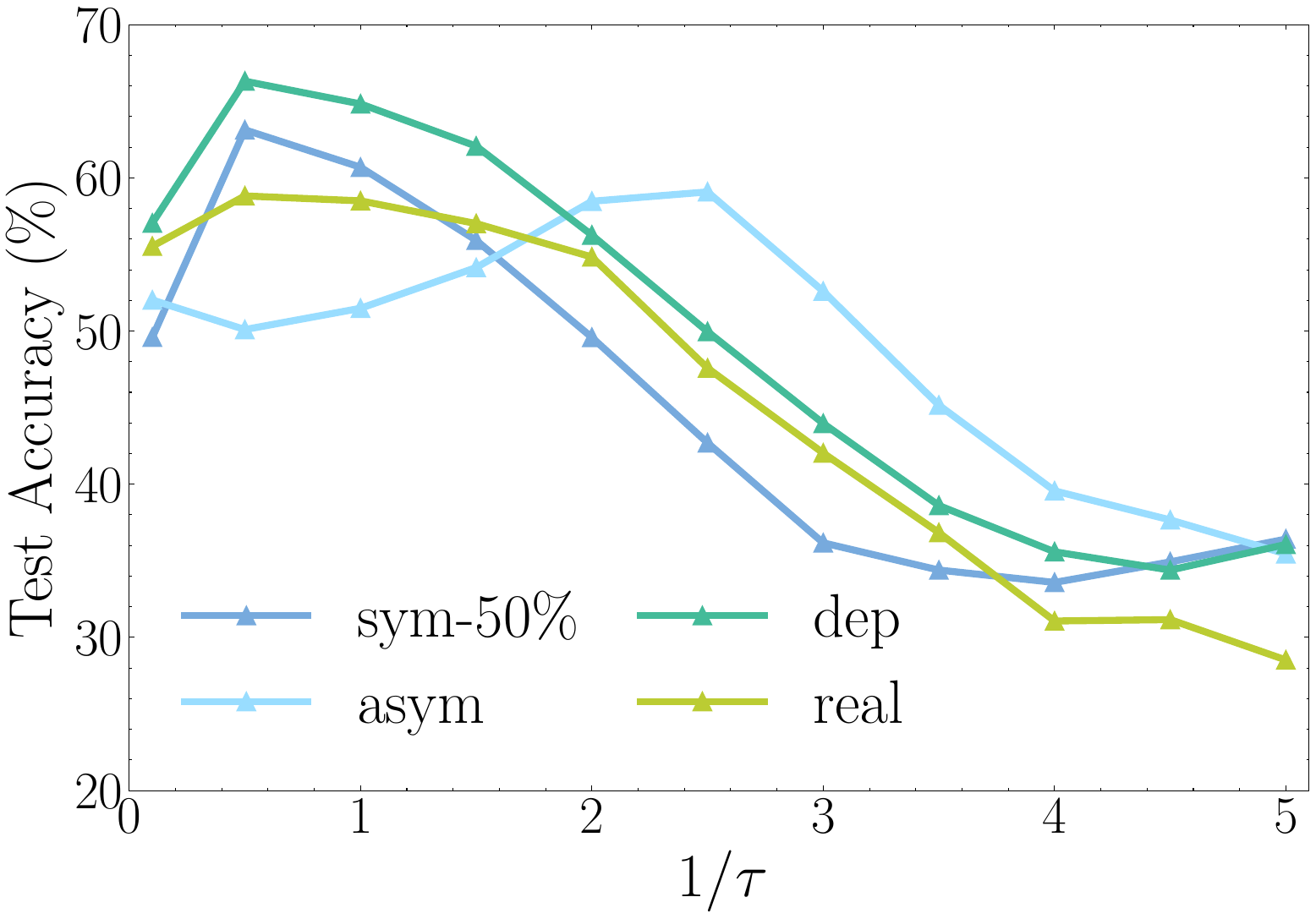}
        \caption{CIFAR-100}
        \label{fig:cifar100_tau}
    \end{subfigure}
     \caption{The Effect of $\tau$ in LogitClip with CIFAR-10 and CIFAR-100 across various noise types.}
     \label{fig:pdf}
    \vspace{-15pt}
\end{figure*}


\paragraph{How does the logit norm threshold $\tau$ affect the noise-robustness of LogitClip?} In Figure~\ref{fig:cifar10_tau} and Figure~\ref{fig:cifar100_tau}, we  ablate how the parameter $\tau$ in our method (\emph{cf.} Eq.~\ref{eq:logitclip}) affects the noise robustness performance. The analysis is based on CIFAR-10 and CIFAR-100 with four types of noisy labels, including symmetric-50$\%$, asymmetric, dependent, and real-world noisy labels. Our results echo the analysis in Proposition~\ref{prop:bound} and Theorem~\ref{tm:sym_robust}, where a smaller $\tau$ would lead to a tighter bound on the difference of the risks between using noisy and clean labels. On the other hand, too small of $\tau$ causes a large lower bound on the loss, which is less desirable from the optimization perspective.  
In Appendix~\ref{app:clean_exp}, We clearly validate the underfitting issue caused by a small $\tau$ with experiments on a clean dataset.

\paragraph{Is LogitClip effective with different architectures?}
To show our proposed method is model-agnostic, we conduct  experiments on a diverse collection of model architectures and present the results in Table \ref{tab:arch}. From the results, we observe that LogitClip consistently improves the test performance on CIFAR-10 when using SqueezeNet~\citep{iandola2016squeezenet}, ResNet~\citep{he2016deep}, DenseNet~\citep{huang2017densely} architectures. For instance, with DenseNet, using LogitClip boosts the test accuracy of CE from 59.34\% to 81.29\%, a \textbf{21.99\%} of direct improvement on CIFAR-10 with Symmetric-50\% noisy labels.

\begin{table*}[h]
\centering
\caption{Average test performance comparison on noisy CIFAR-10 with different network architectures: SqueezeNet~\citep{iandola2016squeezenet}, ResNet~\citep{he2016deep}, DenseNet~\citep{huang2017densely}. All values are percentages. The results are shown as CE / +LC (ours).}
\label{tab:arch}
\renewcommand\arraystretch{0.85}
\resizebox{\textwidth}{!}{
\setlength{\tabcolsep}{3mm}{
\begin{tabular}{c|ccccc}
\toprule
 Architecture  & Sym-20\% & Sym-50\%  & Asymmetric & Dependent & Real\\
\midrule
 SqueezeNet   & 80.77 / \textbf{81.05} & 54.73 / \textbf{71.64}  &  75.67 / \textbf{78.98}  &  75.43 / \textbf{76.74}  & 61.28 / \textbf{75.38}  \\
 ResNet-34  & 74.16 / \textbf{84.88} & 54.48 / \textbf{74.81}  &  73.94 / \textbf{79.69}  &  58.03 / \textbf{76.51}  & 63.24 / \textbf{75.25}  \\
 DenseNet   & 80.40 / \textbf{90.85} & 59.34 / \textbf{81.29}  &  76.35 / \textbf{84.35}  &  62.27 / \textbf{84.20}  & 63.13 / \textbf{78.93}  \\
\bottomrule
\end{tabular}
}
}
\end{table*}

\begin{table*}[!t]
\footnotesize
\centering
\renewcommand\arraystretch{1}
\caption{ Top-1 validation accuracy (\%) on the clean ILSVRC12 validation set of ResNet-18 models trained on WebVision using different loss functions, under the Mini setting \cite{jiang2018mentornet}. The bold indicates the best results. Here, ``Ours" denotes CE equipped with LogitClip and ``Ours+" denotes NCE+AGCE (latest state-of-the-art \citep{zhou2021asymmetric}) equipped with LogitClip.} 
\label{tab:webvision}
\resizebox{0.85\textwidth}{!}{
\setlength{\tabcolsep}{2mm}{
\begin{tabular}{ccccccccc}
\toprule
 Method &  CE &  PHuber-CE & GCE & SCE & NCE+MAE &  NCE+AGCE & \textbf{Ours} & \textbf{Ours+}\\
\midrule
 \textit{best} &62.6&  61.6 & 57.32&  59.52 & 64.08&  63.80 & \textbf{65.12} & 64.92\\
\midrule
 \textit{last} & 60.84 &  59.76 & 53.26& 58.47 & 62.85 &  62.46 & 63.75 & \textbf{64.50}\\
\bottomrule
\vspace{-20pt}

\end{tabular}
}}
\end{table*}

\subsection{WebVision}
Going beyond CIFAR benchmarks, we verify the effectiveness of LogitClip on a large-scale real-world noisy dataset -- WebVision \citep{li2017webvision}. In the WebVision dataset, there are 2.4 million images with real-world noisy labels, crawled from the web (e.g., Flickr and Google) based upon the 1,000 classes of ImageNet ILSVRC12 \citep{deng2009imagenet}. Following the ``Mini" setting used in previous works \citep{jiang2018mentornet, ma2020normalized, zhou2021asymmetric}, we take the first 50 classes of the Google resized image subset. For evaluation, we test the trained networks on the same 50 classes of the ILSVRC12 validation set, which can be seen as a clean validation. For each loss, we train a ResNet-18 network using SGD for 120 epochs with an initial learning rate of 0.1, Nesterov momentum 0.9, weight decay $5 \times 10^{-4}$, and batch size 128. The learning rate is reduced by a factor of 10 after 40 and 80 epochs. We resize the images to $224 \times 224$ and apply the standard data augmentations, including random cropping and random horizontal flip. As shown in Table \ref{tab:webvision}, \textit{best} denotes the score of the epoch where the validation accuracy is optimal, and \textit{last} denotes the scores at the average accuracy in the last 10 epochs. As shown in the table, LogitClip not only outperforms but also enhances existing loss functions by a meaningful margin. The results verify that our method is effective for improving noise-robustness in large-scale real-world scenarios.

\section{Discussion}
\label{sec:discussion}

\paragraph{Relations to existing clipping methods.} In the literature, clipping-based methods have been studied in the context of deep learning \citep{bengio1994learning, zhang2020gradient, abadi2016deep, howard2017mobilenets, sun2021react}. One of the most classic clipping-based methods is gradient clipping, a widely used technique in recurrent neural networks \citep{bengio1994learning}, optimization \citep{hazan2015beyond, levy2016power, zhang2020gradient}, and privacy \citep{abadi2016deep, pichapati2019adaclip}. In the simplest form, gradient clipping is designed to constrain the global parameter gradient norm at a specified threshold. With a loss function $\ell_{\theta}$, the \textit{clipped} gradient with a user-specified threshold $\tau > 0$ can be computed as: 
$$
 \bar{g}_\tau(\theta) \doteq \operatorname{clip}_\tau(g(\theta)), \  \operatorname{clip}_\tau(w) \doteq \begin{cases} \frac{\tau \cdot w}{\|w\|_2} & \text { if }\|w\|_2 \geq \tau \\ w & \text { else }\end{cases},
$$
where $g(\theta)$ denotes the gradient for a mini-batch: $g(\theta) \doteq \frac{1}{b} \sum_{n=1}^b \nabla \ell_\theta\left(x_n, y_n\right)$. Different from gradient clipping which limits the norm of the parameter gradient, our LogitClip method places the constraint directly on the \emph{model output}, i.e., the logit vector. Our method is thus designed to have a \emph{direct and explicit} effect in bounding the loss (\emph{cf.} Theorem~\ref{tm:sym_robust} and Theorem~\ref{tm:asym_robust}) and preventing overfitting to examples with noisy labels.  

 Indeed, recent work \citep{menon2020can} has shown that gradient clipping alone does not endow label noise robustness to neural networks. They instead proposed a noise-robust variant, composite loss-based gradient clipping and the resulting partially Huberised loss (PHuber-CE). The results in Section \ref{sec:experiments} have shown that LogitClip not only outperforms but also enhances the performance of PHuber-CE loss. Our results overall demonstrate the superiority and complementarity of LogitClip to gradient clipping on noise robustness.
 
 ReLU6 \citep{howard2017mobilenets} is a modification of the rectified linear unit (ReLU) to facilitate the learning of sparse features. In particular, ReLU6 clamps the activation of the intermediate layers to a maximum value of 6, $\operatorname{ReLU6}(\boldsymbol{x})=\min (\max (0, \boldsymbol{x}), 6)$. Although ReLU6 provides a constraint on the outputs of the intermediate layers, the resulting loss is unbounded yet since the final output can be multiplied through the last linear layer. The empirical results in Figure~\ref{fig:relu6} show that ReLU6 cannot enhance the robustness to noisy labels while our method (LC-N) achieves a significant improvement. 

\paragraph{Clipping-by-value vs. Clipping-by-norm.} While our logit clipping has demonstrated strong promise in the manner of Clipping-by-norm, one may also ask: can a similar effect be achieved by clipping the logit vector by value? In this ablation, we show that directly constraining the maximum and minimum values of the logit vector does not work well as our method. In particular, we consider the link function as softmax with Clipping-by-value:
\begin{align*}
\label{eq:clip_value}
\bar{\sigma}^{\prime}_{\lambda}(\boldsymbol{z}) \doteq \sigma(\operatorname{clip}^{\prime}_{\lambda}(\boldsymbol{z})) \quad \operatorname{clip}^{\prime}_{\lambda}(\boldsymbol{z}_{j}) \doteq \begin{cases}\lambda  & \text { if }\boldsymbol{z}_{j} \geq \lambda \\
-\lambda  & \text { if }\boldsymbol{z}_{j} \leq -\lambda \\
\boldsymbol{z}_{j} & \text { else }\end{cases},
\end{align*}

where $\lambda$ denotes the constant threshold. For convenience, we set the maximum and minimum values as $\lambda$ and $-\lambda$, respectively. We search the best $\lambda$ in $\{0.1, 0.5, 1, 1.5, \ldots , 4.5, 5\}$.

Figure~\ref{fig:relu6} presents the performance comparison between our method and the variant of Clipping-by-value, denoted as LC-N and LC-V, respectively. While both the two logit clipping methods improve the robustness of CE against noisy labels, LC-V obtains inferior performance compared to our proposed method, and the gaps become remarkably significant under complex noise settings, \emph{e.g.}, instance-dependent and real-world label noise. From a theoretical perspective, CE with LC-V also satisfies the bound in Proposition~\ref{prop:bound} and is hence applicable for the two bounds in Theorem~\ref{tm:sym_robust} and Theorem~\ref{tm:asym_robust}, which indicates the noise-robustness of this method. Nevertheless, LC-V is suboptimal as the clipping operation may diminish the gradients on the clipped components of the logit vector. Besides, LC-V would also modify the direction of the input vector and even change the final prediction. Overall, we demonstrate that our method is superior to the variant of clipping-by-value.

\paragraph{Relations to LogitNorm.}
A concurrent work \citep{wei2022mitigating} employs logit normalization (LogitNorm) to improve the OOD detection and calibration performance. For all training inputs, the logit vector is normalized to be a unit vector with a
constant norm and the resulting loss is defined as: $
\label{eq:norm_loss}
\mathcal{L}_{\text{logit\_norm}}({f}(\boldsymbol{x};\theta),y) = - \log \frac{e^{{f}_y/(\tau \|\boldsymbol{f}\|)}}{\sum^{k}_{i=1} e^{f_{i}/(\tau \|\boldsymbol{f}\|)}}.
$
Our work bears three critical differences, in terms of the problem setting, methodology, and theory. 

\textbf{(1)} \emph{Problem setting}: LogitNorm focuses on improving the performance of detecting out-of-distribution (OOD) examples during the test time, while our work aims to enhance the robustness against noisy labels in the training stage. The learning tasks are fundamentally different. 

\textbf{(2)} \emph{Methodology}: We propose to clamp the logit vector to ensure it is \emph{upper bounded by a constant}, while LogitNorm enforces the norm of logit vectors to be an \emph{exact constant} for all samples. From a constrained optimization perspective, LogitNorm enforces \emph{equality constraint} on the $L_2$ norm of logit vector, whereas LogitClip enforces \emph{inequality constraint}. 
Hence, LogitClip enforces a less strict objective than LogitNorm. 
Referring to the relationship between Gradient Clipping \citep{abadi2016deep, zhang2020gradient} and Gradient Normalization (NGD) \citep{hazan2015beyond, murray2019revisiting}, LogitClip is a unique method that differs from LogitNorm. 

In Table \ref{tab:norm_comp}, we present the performance comparison of LogitClip and LogitNorm in learning with noisy labels. From the comparison, we find that LogitClip is superior to the LogitNorm in this task, especially in those complicated settings. For example, in the asymmetric setting, LogitClip outperforms the LogitNorm method by a large margin of 5.1\%. Intuitively, LogitNorm can improve the robustness to label noise because it also induces a loss bound. However, it enforces a more strict constraint on all training examples, which may make LogitNorm suboptimal in this task. 

\begin{figure}[t]
    \centering
    \includegraphics[width=0.4\textwidth]{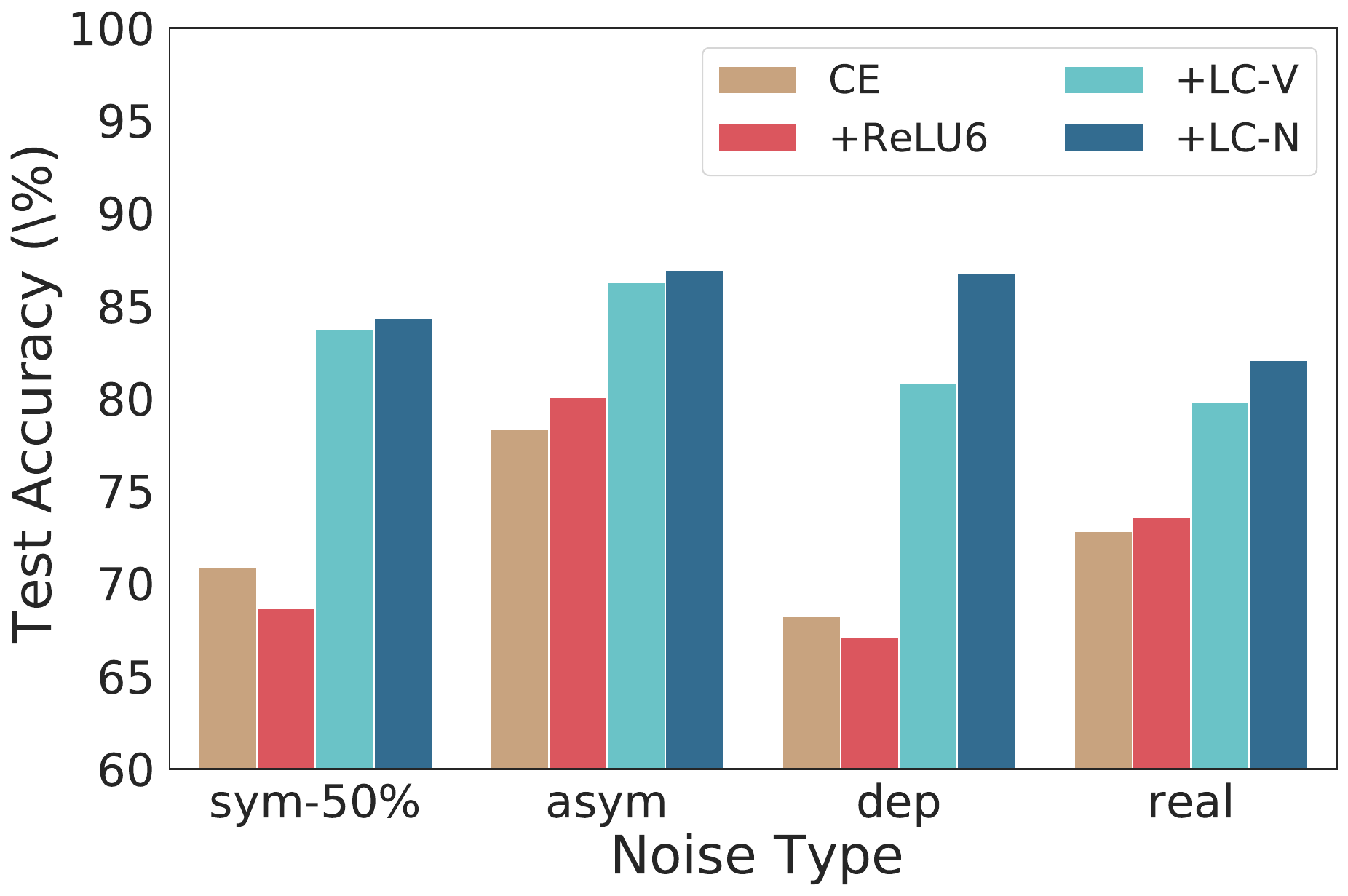}
    \caption{Performance comparison among ReLU6, LC-V, and our method (LC-N) on noisy CIFAR-10.}
    \label{fig:relu6}
\end{figure}

\begin{table}[t]
\footnotesize
\centering
\renewcommand\arraystretch{1.2}
\caption{Comparison between LogitClip and LogitNorm on the CIFAR-10 dataset with various noise settings.} 
\label{tab:norm_comp}
\resizebox{0.45\textwidth}{!}{
\setlength{\tabcolsep}{1mm}{
\begin{tabular}{ccccc}
\toprule
Method & Sym-50\% &	Asymmetric	& Dependent	& Real \\
\midrule
LogitNorm & 83.97 & 81.81 & 84.56 & 80.10 \\
LogitClip & 84.37 & \textbf{86.91} & \textbf{86.74} & \textbf{82.06} \\
\bottomrule
\end{tabular}
}}
\vspace{-10pt}
\end{table}

\textbf{(3)} \emph{Theoretical insight}: LogitNorm aims to decouple the influence of logits’ magnitude from network optimization, and they empirically show that their method leads to more meaningful information to differentiate in-distribution and OOD samples. In contrast, our LogitClip is designed to enforce an upper bound of the resulting loss, as shown in Proposition~\ref{prop:bound}. Furthermore, we provide a theoretical interpretation to further understand why LogitClip introduces the noise-tolerant ability and which kinds of loss functions LogitClip can work with. In summary, our analysis builds a connection between LogitClip and noise robustness, which is novel to the best of our knowledge.

\section{Related Work}
\label{app:relate}

\noindent\textbf{Robust loss functions.} Designing loss functions that are robust to noisy labels has attracted a surge of interest.  \citeauthor{ghosh2017robust} first shows that, for multi-class classification, loss functions that satisfy the symmetric condition, $\sum_{i=1}^{k} \mathcal{L}(f(\mathbf{x}), i)=C, \forall x \in \mathcal{X}, \forall f $, where $C$ is a constant, can be robust to label noise. One of the most classic symmetric loss functions is Mean Absolute Error (MAE): $\mathcal{L}_\text{MAE}\left(\boldsymbol{z}, y\right) = \left\|\boldsymbol{y}-\sigma(\boldsymbol{z})\right\|_{1}$, where $\sigma(\boldsymbol{z})$ denotes the softmax output. Besides, NCE \citep{ma2020normalized} makes any loss function to be symmetric by loss normalization. Despite its theoretical robustness, symmetric losses have been shown to exhibit extremely slow convergence on complicated datasets, since the symmetric condition is too stringent to find a convex loss function \citep{zhang2018generalized, ma2020normalized}.  To alleviate this issue, Generalized Cross Entropy (GCE) \citep{zhang2018generalized} uses a hyperparameter $q$ to balance between MAE and CE, adopting the negative Box-Cox transformation strategy. Taylor Cross Entropy (Taylor-CE) \citep{feng2020can} balances between MAE and CE by controlling the order of the Taylor Series. Partial Huberised Cross Entropy (PHuber-CE) \citep{menon2020can} enhances the noise robustness of CE with a loss variant of gradient clipping. Symmetric Cross Entropy (SCE) \citep{wang2019symmetric} boosts CE symmetrically with a noise-robust counterpart, Reverse Cross Entropy (RCE). Recent work \citep{zhou2021asymmetric} proposes a new family of robust loss functions, termed asymmetric loss functions, which provides a global clean weighted risk when minimizing the noisy risk for any hypothesis class. In this work, our focus is complementary to existing robust loss functions --- we propose a strategy that can universally enhance the noise robustness of existing losses, across various types of label noise.

\paragraph{Other deep learning methods for noise-robustness.} 
In addition to robust loss functions, some other solutions are also
applied to learn with noisy labels \cite{xia2019anchor, li2022estimating, li2022selective, wu2021class2simi, chen2023imprecise, shu2023cmw, shu2021learning, wu2021learning, ding2023improve,zhu2021clusterability, zhu2022beyond, zhu2022detecting, cheng2023mitigating, wei2022aggregate, liu2023humans, wei2022smooth, wei2021when, wei2022crowdsource}, including:
1) Some methods aim to design sample weighting schemes that give higher weights on clean samples \citep{jiang2018mentornet, liu2015classification, ren2018learning, shu2019meta, wei2020metainfonet}. 
2) Some methods propose to train on selected samples, using small-loss selection \citep{han2018co, wei2020combating, yu2019does, Xia2022SampleSW}, GMM distribution \citep{arazo2019unsupervised, li2020dividemix} or (dis)agreement between two models \citep{malach2017decoupling, wei2020combating, yu2019does}. 
3) Loss correction is also a popular direction based on an estimated noise transition matrix \citep{hendrycks2018using, patrini2017making}, or the model's predictions \citep{arazo2019unsupervised, chen2020beyond, reed2014training, tanaka2018joint, zheng2020error}. 
4) Some methods apply regularization techniques to improve generalization under the settings of label noise \citep{fatras2021wasserstein,hu2019simple, xia2020robust, liu2020peer, bai2021pes, liu2020early, zhu2021second, pmlr-v162-liu22w}, such as label smoothing \citep{lukasik2020does, szegedy2016rethinking}, temporal ensembling \citep{laine2016temporal}, and virtual adversarial training \citep{miyato2018virtual}. 
5) Some training strategies for combating noisy labels are built based upon semi-supervised learning methods \citep{li2020dividemix, nguyen2019self} or self-supervised learning \citep{li2022selective}. 
Compared to the above deep learning methods, designing robust loss function are generally a more straightforward and arguably more generic solution with theoretical guarantees.

\section{Conclusion}

In this paper, we propose Logit Clipping (\textbf{LogitClip}), a general strategy that can universally enhance the noise robustness of existing losses, across various types of label noise.
Specifically, we propose to clamp the norm of the logit vector to ensure that it is upper bounded by a constant. 
In this manner, CE loss equipped with our LogitClip method is effectively bounded, alleviating  overfitting to examples with noisy labels. 
As a result, our method could mitigate the undesirable influence of unbounded loss without modifying the loss function.
Moreover, we present theoretical analyses to certify the noise-tolerant ability of this method. Extensive experiments show that LogitClip not only significantly improves the noise robustness of CE loss, but also broadly enhances the generalization performance of popular robust losses.
This method is straightforward to implement with existing losses and can be easily adopted in various practical settings. We hope that our method inspires future theoretical research to explore robust loss from the logit perspective.

\section*{Acknowledgements}

This research is supported by the National Research Foundation, Singapore under its Industry Alignment Fund – Pre-positioning (IAF-PP) Funding Initiative. Lei Feng was supported by the National Natural Science Foundation of China (Grant No. 62106028), Chongqing Overseas Chinese Entrepreneurship and Innovation Support Program. Li is supported in part by the AFOSR Young Investigator Award under No. FA9550-23-1-0184; and faculty research awards/gifts from Google, Meta, and Amazon. We gratefully acknowledge the support of Center for Computational Science and Engineering at Southern University of Science and Technology for our research. Any opinions, findings, conclusions, or recommendations expressed in this material are those of the author(s) and do not reflect the views of the sponsors. 

\newpage

\bibliography{main}
\bibliographystyle{icml2023}

\newpage
\appendix
\onecolumn

\section{Proof of Proposition~\ref{prop:bound}}
\label{app:proofs_1}
\begin{proof}
Give a logit vector $z=f(\boldsymbol{x} ; {\boldsymbol\theta})$, for any class pair $j$ and $k$, we have:
\begin{align*}
    \boldsymbol{z}_{\min} - \boldsymbol{z}_{\max} \leq \boldsymbol{z}_{j} - \boldsymbol{z}_{k} \leq \boldsymbol{z}_{\max} - \boldsymbol{z}_{\min}
\end{align*}

Recall that $-\tau \leq \operatorname{clip}_{\tau}(\boldsymbol{z}_{j}) \leq \tau$, then:
\begin{align*}
    -2\tau \leq \boldsymbol{z}_{j} - \boldsymbol{z}_{k} \leq 2\tau
\end{align*}
Based on Equation~(\ref{eq:ce_logit}), we have:
\begin{align*}
    \log (1 + (K - 1) \cdot e^{-2\tau}) \leq \mathcal{L}^{\tau}_{\mathrm{CE}}\left(f(\boldsymbol{x} ; {\boldsymbol\theta}), y\right)\leq \log (1 + (K - 1) \cdot e^{2\tau}).
\end{align*}
Thus Proposition~\ref{prop:bound} is proved.
\end{proof}

\section{Proof of Theorem~\ref{tm:sym_robust}}
\label{app:proofs_2}

\begin{proof} Recall that for symmetric label noise with noise rate $\eta$, we have: $\eta_{jk} = 1-\eta$ for $j=k$, and $\eta_{jk} = \frac{\eta}{K-1}$. Then, for any model output $f(\boldsymbol{x})$,
\begin{align*}
\mathcal{R}_{\mathcal{L}^{\tau}_{\mathrm{CE}}}^\eta(f) = & \mathbb{E}_{(\boldsymbol{x}, y)\sim\mathcal{P}^\eta_\text{noisy}}\left[\mathcal{L}^{\tau}_{\mathrm{CE}}(f(\boldsymbol{x}), y)\right] \\
=& \mathbb{E}_{\mathcal{P}_{\boldsymbol{x}}}\mathbb{E}_{\mathcal{P}_{y^{\star} \mid \boldsymbol{x}}} \mathbb{E}_{\mathcal{P}_{y \mid y^{\star}}}\left[\mathcal{L}^{\tau}_{\mathrm{CE}}(f(\boldsymbol{x}), y)\right] \\
=& \mathbb{E}_{(\boldsymbol{x}, y^{\star})\sim \mathcal{P}_\text{clean}}\left[(1-\eta) \mathcal{L}^{\tau}_{\mathrm{CE}}(f(\boldsymbol{x}), y^{\star})+\frac{\eta}{K-1} \sum_{j \neq y^{\star}} \mathcal{L}^{\tau}_{\mathrm{CE}}(f(\boldsymbol{x}), j)\right] \\
=&(1-\eta) \mathcal{R}_{\mathcal{L}^{\tau}_{\mathrm{CE}}}(f)+\frac{\eta}{K-1}\left(\sum_{j=1}^K \mathcal{L}^{\tau}_{\mathrm{CE}}(f(\boldsymbol{x}), j)-\mathcal{R}_{\mathcal{L}^{\tau}_{\mathrm{CE}}}(f)\right) \\
=&\left(1-\frac{\eta K}{K-1}\right) \mathcal{R}_{\mathcal{L}^{\tau}_{\mathrm{CE}}}(f)+\frac{\eta}{K-1} \sum_{j=1}^K \mathcal{L}^{\tau}_{\mathrm{CE}}(f(\boldsymbol{x}), j) .
\end{align*}
From proposition \ref{prop:bound}, we have:
$$
K\log (1 + (K - 1) \cdot e^{-2\tau}) \leq \sum_{j}^{K}\mathcal{L}^{\tau}_{\mathrm{CE}}(f(\boldsymbol{x}), j) \leq K\log (1 + (K - 1) \cdot e^{2\tau}).
$$
Thus,
$$
\beta \mathcal{R}_{\mathcal{L}^{\tau}_{\mathrm{CE}}}(f) + \frac{\eta K}{K-1}\log(1+(K-1)e^{-2\tau}) \leq \mathcal{R}_{\mathcal{L}^{\tau}_{\mathrm{CE}}}^\eta(f) \leq 
\beta \mathcal{R}_{\mathcal{L}^{\tau}_{\mathrm{CE}}} + \frac{\eta K}{K-1}\log(1+(K-1)e^{2\tau}).
$$
where $\beta=(1-\frac{\eta K}{K-1})$. We can also write the inequality in terms of $ \mathcal{R}^{\eta}_{\mathcal{L}^{\tau}_{\mathrm{CE}}}(f)$:
$$
\frac{1}{\beta} \left(\mathcal{R}_{\mathcal{L}^{\tau}_{\mathrm{CE}}}^{\eta}(f) - \frac{\eta K}{K-1}\log(1+(K-1)e^{2\tau})\right) \leq \mathcal{R}_{\mathcal{L}^{\tau}_{\mathrm{CE}}}(f) \leq 
\frac{1}{\beta}\left( \mathcal{R}^{\eta}_{\mathcal{L}^{\tau}_{\mathrm{CE}}} - \frac{\eta K}{K-1}\log(1+(K-1)e^{-2\tau})\right).
$$
For $\tilde{f}$, we have:
$$
\mathcal{R}_{\mathcal{L}^{\tau}_{\mathrm{CE}}}(\tilde{f}) - \mathcal{R}_{\mathcal{L}^{\tau}_{\mathrm{CE}}}(f^{\star}) \leq  \frac{1}{\beta}\left( \frac{\eta K}{K-1} \log A^{K}_{\eta}+\mathcal{R}_{\mathcal{L}^{\tau}_{\mathrm{CE}}}^{\eta}(\tilde{f}) - \mathcal{R}^{\eta}_{\mathcal{L}^{\tau}_{\mathrm{CE}}}(f^{\star})\right),
$$
where $A^{K}_{\eta} = \left(\frac{1+(K-1)e^{2\tau}}{1+(K-1)e^{-2\tau}}\right)$. Since $\eta \leq 1-\frac{1}{K}$, $f^{\star}$ is the global minimizer of $\mathcal{R}_{\mathcal{L}^{\tau}_{\mathrm{CE}}}(\tilde{f})$ and $\tilde{f}$ is the global minimizer of $\mathcal{R}^{\eta}_{\mathcal{L}^{\tau}_{\mathrm{CE}}}(\tilde{f})$, we have
$$
 0 \leq \mathcal{R}_{\mathcal{L}^{\tau}_{\mathrm{CE}}}(\tilde{f}) - \mathcal{R}_{\mathcal{L}^{\tau}_{\mathrm{CE}}}(f^{\star}) \leq \frac{\eta K}{(1-\eta)K-1} \cdot A^{K}_{\tau},
$$
which concludes the proof.
\end{proof}

\section{Proof of Theorem~\ref{tm:asym_robust}}
\label{app:proofs_3}

\begin{proof}
For asymmetric label noise, we have
\begin{align*}
\mathcal{R}_{\mathcal{L}^{\tau}_{\mathrm{CE}}}^\eta(f) = & \mathbb{E}_{(\boldsymbol{x}, y)\sim\mathcal{P}^\eta_\text{noisy}}\left[\mathcal{L}^{\tau}_{\mathrm{CE}}(f(\boldsymbol{x}), y)\right] \\
=& \mathbb{E}_{(\boldsymbol{x}, y^{\star})\sim \mathcal{P}_\text{clean}}\left[(1-\eta_{i})\mathcal{L}^{\tau}_{\mathrm{CE}}(f(\boldsymbol{x}), y^{\star})\right] + \mathbb{E}_{(\boldsymbol{x}, y^{\star})\sim \mathcal{P}_\text{clean}}\left[\sum_{j\neq y^{\star}} \eta_{ij} \mathcal{L}^{\tau}_{\mathrm{CE}}(f(\boldsymbol{x}), j)\right]\\
\leq&  \mathbb{E}_{(\boldsymbol{x}, y^{\star})\sim \mathcal{P}_\text{clean}} \left[(1-\eta_{i}) \left(K\log (1 + (K - 1) \cdot e^{2\tau}) - \sum_{j \neq y^{\star}} \mathcal{L}^{\tau}_{\mathrm{CE}}(f(\boldsymbol{x}), j)\right)\right] \\
&+ \mathbb{E}_{(\boldsymbol{x}, y^{\star})\sim \mathcal{P}_\text{clean}}\left[\sum_{j\neq y^{\star}} \eta_{ij} \mathcal{L}^{\tau}_{\mathrm{CE}}(f(\boldsymbol{x}), j)\right] \\
=& K\log (1 + (K - 1) \cdot e^{2\tau}) \cdot \mathbb{E}_{(\boldsymbol{x}, y^{\star})\sim \mathcal{P}_\text{clean}} (1-\eta_{i}) -\mathbb{E}_{(\boldsymbol{x}, y^{\star})\sim \mathcal{P}_\text{clean}}\left[\sum_{j \neq y} \lambda_j \mathcal{L}^{\tau}_{\mathrm{CE}}(f(\boldsymbol{x}), j)\right],
\end{align*}
where $\lambda_j = (1-\eta_i-\eta_{ij})$. 
On the other hand, we have
$$
\mathcal{R}_{\mathcal{L}^{\tau}_{\mathrm{CE}}}^\eta(f) \geq K\log (1 + (K - 1) \cdot e^{-2\tau}) \cdot \mathbb{E}_{(\boldsymbol{x}, y^{\star})\sim \mathcal{P}_\text{clean}} (1-\eta_{i}) -\mathbb{E}_{(\boldsymbol{x}, y^{\star})\sim \mathcal{P}_\text{clean}}\left[\sum_{j \neq y^{\star}} \lambda_j \mathcal{L}^{\tau}_{\mathrm{CE}}(f(\boldsymbol{x}), j)\right]
$$
Hence,
$$
\mathcal{R}_{\mathcal{L}^{\tau}_{\mathrm{CE}}}^\eta\left(f^\star\right)-\mathcal{R}_{\mathcal{L}^{\tau}_{\mathrm{CE}}}^\eta(\tilde{f}) \leq B^{K}_{\tau} + 
\mathbb{E}_{(\boldsymbol{x}, y^{\star})\sim \mathcal{P}_\text{clean}}\left[\sum_{j \neq y^{\star}} \lambda_j\left(\mathcal{L}^{\tau}_{\mathrm{CE}}\left(\tilde{f}_j(\boldsymbol{x}), j\right)-\left(\mathcal{L}^{\tau}_{\mathrm{CE}}\left(f^*(\boldsymbol{x}), j\right)\right)\right)\right]
$$
where $B^{K}_{\tau} = K\log\left(\frac{1+(K-1)e^{2\tau}}{1+(K-1)e^{-2\tau}}\right) \mathbb{E}_{p(\boldsymbol{x}, y)}\left(1-\eta_i\right)$. Let $q^{-}$ and $q^{+}$ denote the lower bound and the upper bound in Proposition \ref{prop:bound}. We assume $\mathcal{R}_{\mathcal{L}^{\tau}_{\mathrm{CE}}}^\eta\left(f^\star\right) = q^{-}$, i.e., $\mathcal{L}^{\tau}_{\mathrm{CE}}\left(f^\star(\boldsymbol{x}), y^{\star}\right)= q^{-}$, which is only satisfied iff $f^{\star}_{j}(\boldsymbol{x})=\tau$ when $j=y^\star$ and $f^{\star}_{j}(\boldsymbol{x})=-\tau$ when $j \neq y^{\star}$. In this case, $\mathcal{L}^{\tau}_{\mathrm{CE}}\left(f^\star(\boldsymbol{x}), j\right) = q^{+}, \forall j \neq y^{\star}$ and $\mathcal{L}^{\tau}_{\mathrm{CE}}\left(f^\star(\boldsymbol{x}), j\right) \leq q^{+}, \forall j \in [K]$. Since $f^{\star}$ is the global minimizer of $\mathcal{R}_{\mathcal{L}^{\tau}_{\mathrm{CE}}}(f)$ and $\lambda=\left(1-\eta_i-\eta_{i j}\right)>0$, we have
$$
\mathbb{E}_{(\boldsymbol{x}, y^{\star})\sim \mathcal{P}_\text{clean}}\left[\sum_{j \neq y^{\star}} \lambda_j\left(\mathcal{L}^{\tau}_{\mathrm{CE}}(\tilde{f}(\boldsymbol{x}), j)-\mathcal{L}^{\tau}_{\mathrm{CE}}\left(f^\star(\boldsymbol{x}), j\right)\right)\right] \leq 0.
$$
Thus, we have
$$
0 \leq \mathcal{R}_{\mathcal{L}}^\eta\left(f^\star\right)-\mathcal{R}_{\mathcal{L}}^\eta(\tilde{f}) \leq B^{K}_{\tau},
$$
which concludes the proof.
\end{proof}

\section{Proof of Proposition~\ref{prop:lip_condition}}
\label{app:proofs_4}

\begin{proof}
Let $\boldsymbol{p} = \sigma^{\prime}_{\tau}(\boldsymbol{z})$. Recall that $-\tau\leq\boldsymbol{z}_j\leq\tau, \forall j \in [K]$, we have:
$$
\frac{1}{1+(K-1)\cdot e^{2\tau}} \leq \boldsymbol{p}_j \leq \frac{1}{1+(K-1)\cdot e^{-2\tau}}, \forall j \in [K],
$$

Let $M^{K}_{\tau}$ and $N^{K}_{\tau}$ denote the lower and upper bound of $\boldsymbol{p}_{j}$.
Given that the base loss $\phi(\boldsymbol{p}_{y})$ satisfies the Lipschitz condition with constant $L$ on the domain. For any $\boldsymbol{p}_{y} \in [M^{K}_{\tau}, N^{K}_{\tau}]$, we have
$$
|\phi(\boldsymbol{p}_{y})-\phi(M^{K}_{\tau})| \leq L|\boldsymbol{p}_{y}-M^{K}_{\tau}| \leq L|N^{K}_{\tau}-M^{K}_{\tau}|,
$$
and
\begin{align*}
|\phi(\boldsymbol{p}_{y})| &=|\phi(\boldsymbol{p}_{y})-\phi(M^{K}_{\tau})+\phi(M^{K}_{\tau})| \\
& \leq|\phi(\boldsymbol{p}_{y})-\phi(M^{K}_{\tau})|+|\phi(M^{K}_{\tau})| \\
& \leq L|N^{K}_{\tau}-M^{K}_{\tau}|+|\phi(M^{K}_{\tau})| . 
\end{align*}

Since $N^{K}_{\tau}-M^{K}_{\tau} \geq 0, \forall \tau > 0$, we have:
$$
\left|\mathcal{L}^{\tau}_{\phi}\left(f(\boldsymbol{x} ; {\boldsymbol\theta}), y\right)\right| \leq L\left(N^{K}_{\tau} - M^{K}_{\tau}\right) + \left|\phi(M^{K}_{\tau})\right|,
$$
which concludes the proof.

\end{proof}

\section{Theoretical analysis under instance-dependent setting}
\label{app:theory_dependent}

In Theorem \ref{tm:sym_robust} and Theorem \ref{tm:asym_robust}, we have provably shown the noise-tolerant ability of cross-entropy loss with LogitClip. Here, we extend the theoretical analysis to an \emph{instance-dependent setting}, where the noise rate $\eta_{\boldsymbol{x}}$ is a function of instance $\boldsymbol{x}$ and $\eta_{\boldsymbol{x}j}$ may vary across classes $j$.

\begin{theorem}
\label{tm:dep_robust}
Under instance-dependent label noise with $1-\eta_{\boldsymbol{x}} >  \eta_{\boldsymbol{x}j} , \forall \boldsymbol{x}, j \neq y_{\boldsymbol{x}}^{\star}$, where $\eta_{\boldsymbol{x} j}=p(y=j \mid \boldsymbol{x}), \forall j \neq i$ and $\left.\left(1-\eta_{\boldsymbol{x}}\right)=p(y=i \mid \boldsymbol{x}, y^{\star}=i)\right)$, then
$$
0 \leq \mathcal{R}_{\mathcal{L}^{\tau}_{\mathrm{CE}}}^\eta\left(f^*\right)-\mathcal{R}_{\mathcal{L}^{\tau}_{\mathrm{CE}}}^\eta(\tilde{f}) \leq C^{K}_{\tau}
$$
where $C^{K}_{\tau} = K\log\left(\frac{1+(K-1)e^{2\tau}}{1+(K-1)e^{-2\tau}}\right) \mathbb{E}_{(\boldsymbol{x}, y^{\star})\sim \mathcal{P}_\text{clean}}\left(1-\eta_{\boldsymbol{x}}\right)>0$.
\end{theorem}

\begin{proof}
For instance-dependent label noise, we have
\begin{align*}
\mathcal{R}_{\mathcal{L}^{\tau}_{\mathrm{CE}}}^\eta(f) = & \mathbb{E}_{(\boldsymbol{x}, y)\sim\mathcal{P}^\eta_\text{noisy}}\left[\mathcal{L}^{\tau}_{\mathrm{CE}}(f(\boldsymbol{x}), y)\right] \\
=& \mathbb{E}_{\mathcal{P}_{\boldsymbol{x}}}\mathbb{E}_{\mathcal{P}_{y^{\star} \mid \boldsymbol{x}}} \mathbb{E}_{\mathcal{P}_{y \mid \boldsymbol{x}, y^{\star}}}\left[\mathcal{L}^{\tau}_{\mathrm{CE}}(f(\boldsymbol{x}), y)\right] \\
=& \mathbb{E}_{(\boldsymbol{x}, y^{\star})\sim \mathcal{P}_\text{clean}}\left[(1-\eta_{\boldsymbol{x}})\mathcal{L}^{\tau}_{\mathrm{CE}}(f(\boldsymbol{x}), y^{\star})\right] + \mathbb{E}_{(\boldsymbol{x}, y^{\star})\sim \mathcal{P}_\text{clean}}\left[\sum_{j\neq y^{\star}} \eta_{\boldsymbol{x}j} \mathcal{L}^{\tau}_{\mathrm{CE}}(f(\boldsymbol{x}), j)\right]\\
\leq&  \mathbb{E}_{(\boldsymbol{x}, y^{\star})\sim \mathcal{P}_\text{clean}} \left[(1-\eta_{\boldsymbol{x}}) \left(K\log (1 + (K - 1) \cdot e^{2\tau}) - \sum_{j \neq y^{\star}} \mathcal{L}^{\tau}_{\mathrm{CE}}(f(\boldsymbol{x}), j)\right)\right] \\
&+ \mathbb{E}_{(\boldsymbol{x}, y^{\star})\sim \mathcal{P}_\text{clean}}\left[\sum_{j\neq y^{\star}} \eta_{\boldsymbol{x}j} \mathcal{L}^{\tau}_{\mathrm{CE}}(f(\boldsymbol{x}), j)\right] \\
=& K\log (1 + (K - 1) \cdot e^{2\tau}) \cdot \mathbb{E}_{(\boldsymbol{x}, y^{\star})\sim \mathcal{P}_\text{clean}} (1-\eta_{\boldsymbol{x}}) -\mathbb{E}_{(\boldsymbol{x}, y^{\star})\sim \mathcal{P}_\text{clean}}\left[\sum_{j \neq y} \lambda_{\boldsymbol{x}j} \mathcal{L}^{\tau}_{\mathrm{CE}}(f(\boldsymbol{x}), j)\right],
\end{align*}
where $\lambda_j = (1-\eta_i-\eta_{\boldsymbol{x}j})$. 
On the other hand, we have
$$
\mathcal{R}_{\mathcal{L}^{\tau}_{\mathrm{CE}}}^\eta(f) \geq K\log (1 + (K - 1) \cdot e^{-2\tau}) \cdot \mathbb{E}_{(\boldsymbol{x}, y^{\star})\sim \mathcal{P}_\text{clean}} (1-\eta_{\boldsymbol{x}}) -\mathbb{E}_{(\boldsymbol{x}, y^{\star})\sim \mathcal{P}_\text{clean}}\left[\sum_{j \neq y^{\star}} \lambda_{\boldsymbol{x}j} \mathcal{L}^{\tau}_{\mathrm{CE}}(f(\boldsymbol{x}), j)\right]
$$
Hence,
$$
\mathcal{R}_{\mathcal{L}^{\tau}_{\mathrm{CE}}}^\eta\left(f^\star\right)-\mathcal{R}_{\mathcal{L}^{\tau}_{\mathrm{CE}}}^\eta(\tilde{f}) \leq C^{K}_{\tau} + 
\mathbb{E}_{(\boldsymbol{x}, y^{\star})\sim \mathcal{P}_\text{clean}}\left[\sum_{j \neq y^{\star}} \lambda_{\boldsymbol{x}j}\left(\mathcal{L}^{\tau}_{\mathrm{CE}}\left(\tilde{f}_j(\boldsymbol{x}), j\right)-\left(\mathcal{L}^{\tau}_{\mathrm{CE}}\left(f^*(\boldsymbol{x}), j\right)\right)\right)\right]
$$
where $C^{K}_{\tau} = K\log\left(\frac{1+(K-1)e^{2\tau}}{1+(K-1)e^{-2\tau}}\right) \mathbb{E}_{(\boldsymbol{x}, y^{\star})\sim \mathcal{P}_\text{clean}}\left(1-\eta_{\boldsymbol{x}}\right)$. From the proof of Theorem \ref{tm:asym_robust}, we have $\mathcal{L}^{\tau}_{\mathrm{CE}}(\tilde{f}(\boldsymbol{x}), j)-\mathcal{L}^{\tau}_{\mathrm{CE}}\left(f^\star(\boldsymbol{x}), j\right) \leq 0, \forall \boldsymbol{x}, j \neq y_{\boldsymbol{x}}^{\star}$. Recall that $\lambda=\left(1-\eta_{\boldsymbol{x}}-\eta_{\boldsymbol{x} j}\right)>0$, we have 
$$
\mathbb{E}_{p(\boldsymbol{x}, y)}\left[\sum_{j \neq y^{\star}} \lambda_{\boldsymbol{x}j}\left(\mathcal{L}^{\tau}_{\mathrm{CE}}(\tilde{f}(\boldsymbol{x}), j)-\mathcal{L}^{\tau}_{\mathrm{CE}}\left(f^\star(\boldsymbol{x}), j\right)\right)\right] \leq 0.
$$
Thus, we have
$$
0 \leq \mathcal{R}_{\mathcal{L}}^\eta\left(f^\star\right)-\mathcal{R}_{\mathcal{L}}^\eta(\tilde{f}) \leq C^{K}_{\tau},
$$
which concludes the proof.
\end{proof}

\section{More details on experimental setup}
\label{app:exp_setup}

\paragraph{Hyperparameter setting.} We conduct all the experiments on NVIDIA GeForce RTX 3090, and implement all methods by PyTorch. We tune the hyperparameters for all compared methods and find that the optimal settings basically match those in their original papers. Specifically, for GCE, we set $q=0.7$. For SCE, we set $\alpha=0.5$ and $\beta=1.0$. For AEL loss, we set $a=2.5$. For AUL loss, we set $a=5.5$ and $q=3$. For PHuber-CE, we set $\tau=10$ for CIFAR-10 and $\tau=30$ for CIFAR-100 and WebVision. For the experiments of NCE+MAE on CIFAR-100 and WebVision, we set $\alpha=50$ and $\beta=1$. For NCE+AGCE on CIFAR-100, we set $\alpha=50$, $\beta=0.1$, $a=1.8$ and $q=3.0$. On WebVision, we set the hyperparameters of NCE+AGCE as $\alpha=50$, $\beta=0.1$, $a=2.5$, and $q=3.0$. For the best $\tau$ of our LogitClip, it may depend on the dataset, noise type, and the base loss. In Table~\ref{tab:param}, we present the best values of $1/\tau$ in CE with LogitClip.
\begin{table*}[h]
\footnotesize
\centering
\renewcommand\arraystretch{1}
\vspace{0.4cm}
\caption{ Best Values of $1/\tau$ for CE+LogitClip on different datasets with various noise settings.} 
\label{tab:param}
\resizebox{0.7\textwidth}{!}{
\setlength{\tabcolsep}{2mm}{
\begin{tabular}{cccccc}
\toprule
 Dataset &  Symmetric-20\% &  Symmetric-50\% & Asymmetric & Dependent & Real-world \\
\midrule
 CIFAR-10 &1.0&  1.5 & 2.5 &  2.0 & 2.5 \\
\midrule
 CIFAR-100 & 0.5 &  0.5 & 2.5 & 0.5 & 0.5 \\
 \midrule
 WebVision & \multicolumn{5}{c}{1.2} \\
\bottomrule
\end{tabular}
}}
\end{table*}

\section{More empirical results}
\label{app:emp_result}

\subsection{Can LogitClip improve deep learning methods?}
\label{app:dividemix}

In the experiments shown in Section \ref{sec:experiments}, we show that our method can consistently improve the noise robustness of existing popular losses, including non-robust losses and robust losses. One may raise the question: Can LogitClip improve deep learning methods? Here, we use DivideMix \citep{li2020dividemix} as the representative method to show the universality of our LogitClip. For the experiments with DivideMix, we use the same setting as those reported in the paper of DivideMix. Specifically, we use an 18-layer PreAct Resnet \citep{he2016deep} and train it using SGD with a momentum of 0.9, a weight decay of 0.0005, and a batch size of 128. The network is trained for 300 epochs. The warm-up period is 10 epochs for CIFAR-10. For the hyperparameters, we set $M=2, T=0.5$, $\alpha=4$, and $\tau=0.5$. For DivideMix + LogitClip, we employ the logit clipping for all the model outputs during training. The test performance for DivideMix on noisy CIFAR-10 is reported in Table~\ref{tab:dividemix}. From the results, we can observe that our LogitClip can consistently improve the performance of DivideMix with a meaningful margin, which validates the universality of our method in boosting noise robustness.

\begin{table*}[!h]
\footnotesize
\centering
\renewcommand\arraystretch{1}
\vspace{0.4cm}
\caption{ Test performance comparison for DivideMix \citep{li2020dividemix} on noisy CIFAR-10 with different noisy types. The results show that our method can boost the performance of DivideMix.} 

\label{tab:dividemix}
\resizebox{0.7\textwidth}{!}{
\setlength{\tabcolsep}{3mm}{
\begin{tabular}{ccccc}
\toprule
 Method &   Symmetric-50\% & Asymmetric & Dependent & Real-world \\
\midrule
DivideMix      & 94.41 & 92.02 & 94.11 & 92.24 \\
+LC(Ours) & \textbf{95.15} & \textbf{92.73} & \textbf{95.25} & \textbf{93.16} \\
\bottomrule
\end{tabular}
}}
\end{table*}

\newpage
\subsection{Logit Clipping vs. Norm Regularization}
\label{app:normreg}

As demonstrated in Subsection~\ref{sec:method}, our training objective can be formalized as a constrained optimization with inequality constraint. Therefore, we may consider an alternative method by simply adding the constraint via the Lagrangian multiplier, termed Norm Regularization:
\begin{equation*}
\mathcal{L}_{\text{logit\_penalty}}(f(\boldsymbol{x};{\theta}),y) = 
\mathcal{L}_{\text{CE}}(f(\boldsymbol{x};{\theta}),y) + \lambda {\|f(\boldsymbol{x};{\theta})\|}_2.
\end{equation*}
In the experiments, we select the best $\lambda$ in $\{0.01, 0.05, 0.1, 0.5\}$. Our results in Figure~\ref{fig:reg} show that Norm Regularization is inferior to our LogitClip across four noise types, while both methods improve the test accuracy compared to Cross Entropy loss. With Norm Regularization, we notice that the trained network can suffer from optimization difficulty and sometimes fail to converge if $\lambda$ is too large (which is needed to regularize the logit norm effectively). Overall, we show that simply constraining the logit norm during training cannot achieve comparable performance as our LogitClip, which significantly improves the noise robustness.
\begin{figure}[h]
    \centering
    \includegraphics[width=0.32\textwidth]{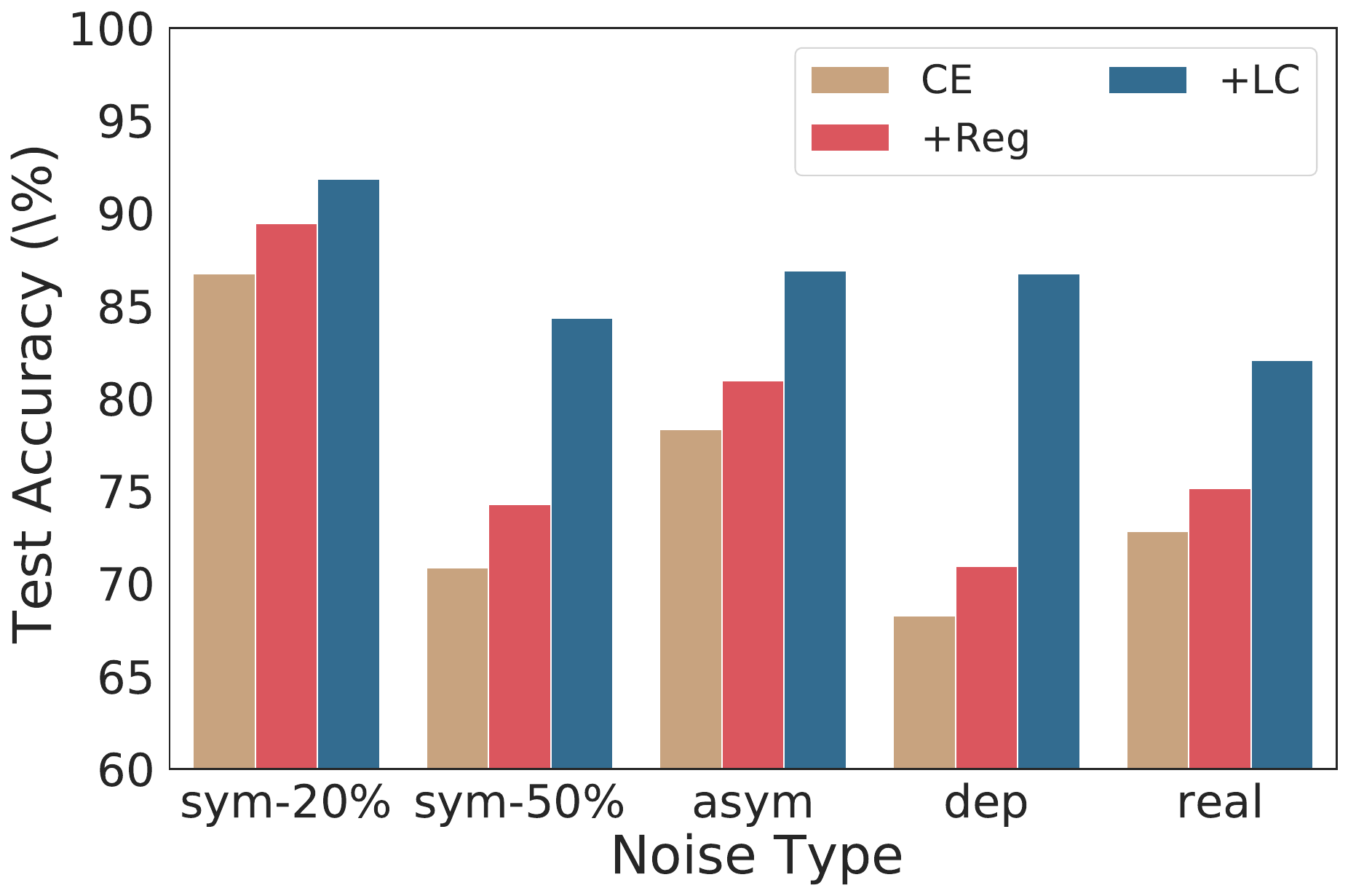}
    \caption{Test performance comparison among CE, Norm Regularization (+Reg), and our method (+LC) on noisy CIFAR-10 across different noise types.}
    \label{fig:reg}
    \vspace{-15pt}
\end{figure}

\subsection{Performance on Clean datasets}
\label{app:clean_exp}

We evaluate the performance of LogitClip on the clean CIFAR-100 dataset with different $\tau$. As discussed in Section~\ref{sec:logitclip} (Proposition 3.1), LogitClip is equivalent to vanilla cross-entropy loss if the $\tau$ is sufficiently large. Contrastively, a small $\tau$ will induce a large lower bound on the loss value, which may lead to difficulty in loss optimization (Underfitting). As shown in Table~\ref{tab:clean_param}, LogitClip with a large $\tau$ achieves comparable performance as the vanilla CE loss. Besides, the performance of LogitClip can be degraded as we decrease the value of $\tau$, which validates the underfitting issue in our analysis.

\begin{table*}[h]
\footnotesize
\centering
\renewcommand\arraystretch{1}
\caption{Performance of CE+LogitClip under different $\tau$ on the clean CIFAR-100 dataset.} 
\label{tab:clean_param}
\resizebox{0.7\textwidth}{!}{
\setlength{\tabcolsep}{2mm}{
\begin{tabular}{cccccccc}
\toprule
$\tau$ &	20 &	10 & 2 & 1 & 0.5 & 0.25 & CE \\
\midrule
Accuracy & \textbf{75.90} & 75.51 & 74.38 & 73.23 & 68.00 & 46.04 & \textbf{75.98}\\
\bottomrule
\end{tabular}
}}
\end{table*}




\subsection{Logit Clipping can improve the latest SOTA methods}
\label{app:sota}
In this paper, we mainly focus on improving existing robust losses. Besides, we show our method can also enhance some other deep learning methods by improving DivideMix (see Appendix~\ref{app:dividemix}). In the tables~\ref{tab:moresota}, we provide empirical evidence that our method can meaningfully improve SOP \citep{liu2022sop} and SAM \citep{foret2021sharpnessaware}. We use the same settings as Section \ref{sec:experiments} of our paper for SAM, while adopting the settings reported in their paper for SOP. These results further demonstrate the complementarity of LogitClip with previous techniques.

\begin{table*}[!h]
\footnotesize
\centering
\renewcommand\arraystretch{1}
\vspace{0.4cm}
\caption{ Test performance comparison for SOP \citep{liu2022sop} and SAM \citep{foret2021sharpnessaware} on noisy CIFAR-10/100 with different noisy types. The results show that our method can boost the performance of SAM and SOP.} 

\label{tab:moresota}
\resizebox{0.8\textwidth}{!}{
\setlength{\tabcolsep}{1mm}{
\begin{tabular}{ccccc}
\toprule
 Method &   CIFAR10-Sym-50\% & CIFAR10-ASym-40\% & CIFAR100-Sym-50\% & CIFAR100-ASym-40\% \\
\midrule
SOP      & 88.31 & 84.43 & 61.68 & 66.89 \\
+LC(Ours) & \textbf{89.20} & \textbf{86.03} & \textbf{63.85} & \textbf{70.41} \\
SAM      & 86.16 & 91.60 & 55.23 & 51.90 \\
+LC(Ours) & \textbf{89.29} & \textbf{91.80} & \textbf{67.17} & \textbf{72.08} \\
\bottomrule
\end{tabular}
}}
\end{table*}

\end{document}